%% file: main.tex
\documentclass[sn-nature,pdflatex]{sn-jnl} 
\usepackage{geometry}
 \usepackage{layout}

\input{_custom_added_preamble}

\begin{document}
%
\includecomment{methods_section}
\includecomment{code_availability_section}
\includecomment{acknowledgements_section} 
\includecomment{pointer_to_suppl_info_arxiv} 
%
\excludecomment{pointer_to_suppl_info_submission} 
\includecomment{supplementary_information_section}

\title[IBNN: implicit bias neural network]{Updating the standard neuron model in artificial neural networks}

\author*[1]{\fnm{Raul} \sur{Mohedano}}\email{raul.mohedano@csic.es}
\author[2]{\fnm{Thomas} \sur{Batard}}
\author[1,3]{\fnm{Erik} \sur{Velasco-Salido}}
\author[2]{\fnm{Ramsses} \sur{De Los Santos Mendoza}}
\author[2]{\fnm{Jorge H.} \sur{Mart{\'\i}nez}}
\author[4]{\fnm{Stacey} \sur{Levine}}
\author*[1]{\fnm{Marcelo} \sur{Bertalm{\'\i}o}}\email{marcelo.bertalmio@csic.es}

\affil[1]{\orgname{Spanish National Research Council (CSIC)}, \orgaddress{\city{Madrid}, \country{Spain}}}
\affil[2]{\orgname{Center for Research in Mathematics (CIMAT)}, \orgaddress{\city{Guanajuato}, \country{Mexico}}}
\affil[3]{\orgname{Universidad Aut{\'o}noma de Madrid (UAM)}, \orgaddress{\city{Madrid}, \country{Spain}}}
\affil[4]{\orgname{National Science Foundation (NSF)}, \orgaddress{\city{Alexandria}, \state{VA}, \country{USA}}}

\abstract{
    From their inception in the 1950s, artificial neural networks (ANNs) started using the so-called point neuron model then prevalent in neuroscience, hoping that this analogy would allow for a better emulation of brain function. Over the years the neuroscience literature has shown that the point neuron model is too simplistic to properly represent many fundamental neural processes; however, the standard neuron model in ANNs still remains the same. Here we substitute it by a very recent model of cortical cells and demonstrate through theoretical analyses and experimental results how, simply by using a more realistic neural unit element without augmenting the number of parameters, the resulting ANNs offer a number of important advantages that include increases in expressivity, robustness and learning speed, and a reduction in memorization and the amount of training data needed.
}


\maketitle

\section*{Introduction} 

In the neuroscience of the mid 20th century, the prevalent view on neural computation was a simplistic abstraction \cite{McCulloch1943,Hubel1959} currently known as the point neuron model, where it is assumed that each individual neuron linearly sums its inputs, i.e. that its dendrites are passive and act as mere ``readers'' of incoming signals, and after this linear filtering there follows a nonlinear process to generate the neuronal output. 
This approach was adopted to represent neurons in \acfp{ANN} 
when they were created in the late 1950s \cite{Rosenblatt1958}, 
with the belief that matching the brain starts with emulating its building block, the neuron \cite{Haykin2009}. 
To this day, neural units in \acp{ANN} still follow the same paradigm \cite{Chavlis2025}, and the cascade of linear-nonlinear stages in \acp{ANN} is referred to as the standard \ac{DL} framework \cite{Bastounis2021}. 

However, the ANN architectures that are currently most popular are poorly aligned with brain functionality \cite{Sartzetaki2024,KarDiCarlo2024,Bowers2022,Tong2024}
and have some critical limitations that do not affect living beings, such as requiring huge amounts of training data and being extremely sensitive to virtually unnoticeable perturbations of their input \cite{Wichmann2023}.
Our conjecture is the following: at the root of these and other fundamental constraints and drawbacks of \acp{ANN} lies the fact that the point neuron model is overly simplistic.
It assumes that dendrites only act linearly, but since the late 1990s and early 2000s it has become clear that dendritic processing is not only much more complex, but also that complex dendritic computations are an indispensable element of key neural behaviors and brain functions \cite{PolegPolsky2019,Pagkalos2024,Poirazi2020}.
Of note, an essential property of dendrites is that they can receive \acp{BAP} initiated from the soma \cite{Stuart2015}, and the importance of \acp{BAP} for the computations of dendrites, single cells and the whole network cannot be overstated \cite{Francioni2022,Stuyt2022,Stingl2025}.
There is currently a renewed interest in the research community in leveraging neuroscience to accelerate progress in \acs{AI} \cite{Zador2023}, and a number of recent works consider more realistic dendritic models to improve the performance of \acp{ANN} \cite{Iyer2022,Chavlis2025}.
However, these approaches still fall under the standard \ac{DL} framework, which is not adequate to represent the back and forth exchange between cell output and dendritic contributions caused by \acp{BAP} \cite{Larkum2022}.

Very recently, a new model for neurons in the primary visual cortex that considers the interaction between dendritic nonlinearities and \acp{BAP} was shown to be able to explain a number of physiological results that challenge standard approaches  \cite{Rentzeperis2025}.
Here, we adapt this more realistic activity model of biological neurons into a novel formulation for the neural units in ANNs. 
Compared with an ANN that uses the standard neuron model and has the same number of parameters,
an ANN built using this updated neuron model is much less impacted by the critical limitations mentioned above, being considerably more robust to input perturbations and requiring a significantly smaller amount of training data,
as demonstrated by our theoretical analyses and experimental results.
Furthermore, the novel neuron model allows the network to have other very relevant advantages as well: to be more expressive, requiring less parameters to achieve a given quality level in the task it solves; to learn faster, reaching a target performance with fewer learning steps; and to memorize less, which is an asset when the training data has noisy or mislabeled samples, as is generally the case.

\section*{Results} 

\subsection*{Updating the neuron model}

We say that an ANN follows the \ac{SM} when the output $u_i$ of a neural unit given an input $\ve{x} \in \R[N]$ may be written as:
\begin{align}
    \label{eq:u}
    u_i &= \phi(y_i),\\
        \label{eq:y}
        y_i &= \sum_{j=1}^N m_{ij} x_j - b_i, 
\end{align}

\noindent where $\ve{m}_i \in \R[N]$ is a linear filter, $b_i \in \R$ is the bias term, and $\phi$ is the (nonlinear) activation function.

\begin{figure}[t]
    \centering
    \begin{subfigure}[t]{0.49\textwidth}
        \centering
        \includegraphics[width=1.0\linewidth]{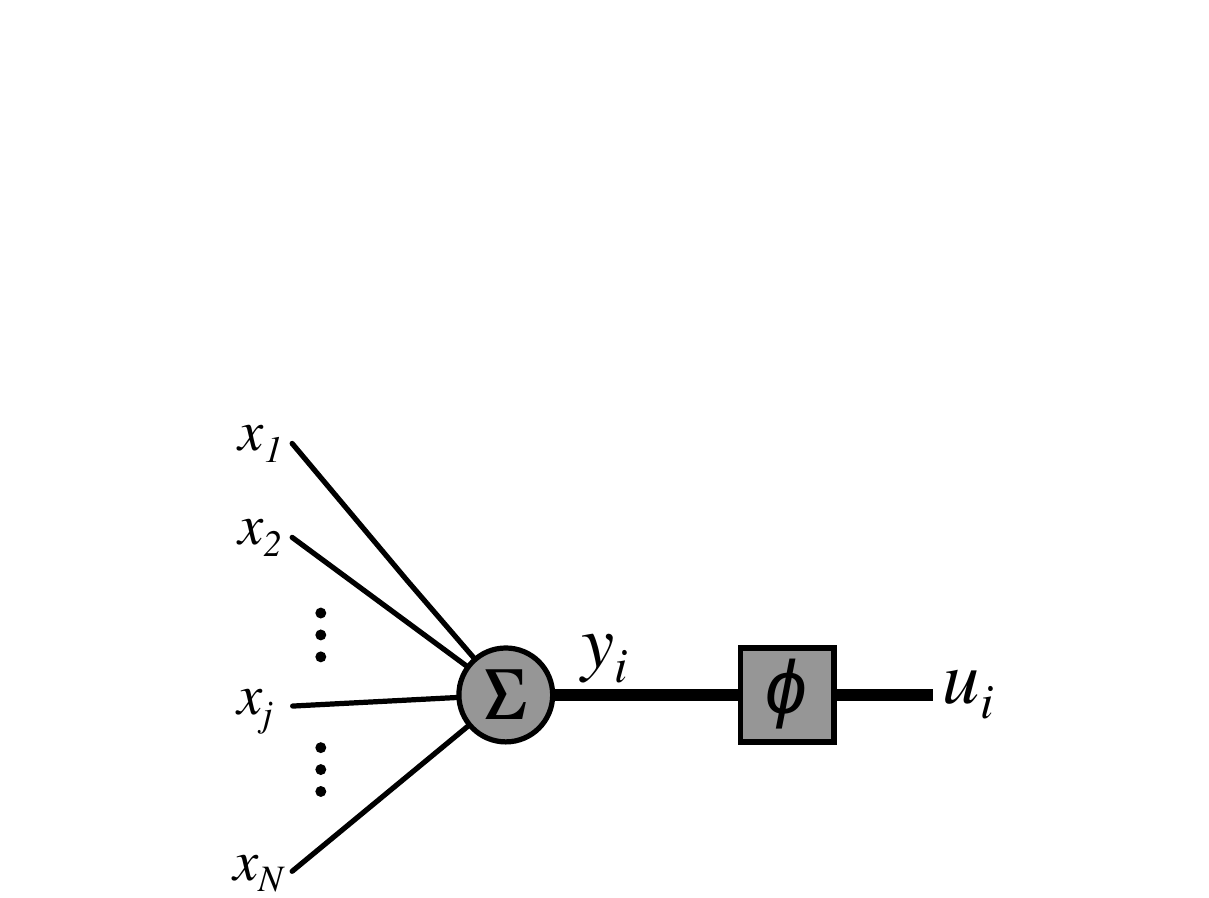}
        \caption{}
        \label{fig:schemeSM}
    \end{subfigure}
    \hfill
    \begin{subfigure}[t]{0.49\textwidth}
        \centering
        \includegraphics[width=1.0\linewidth]{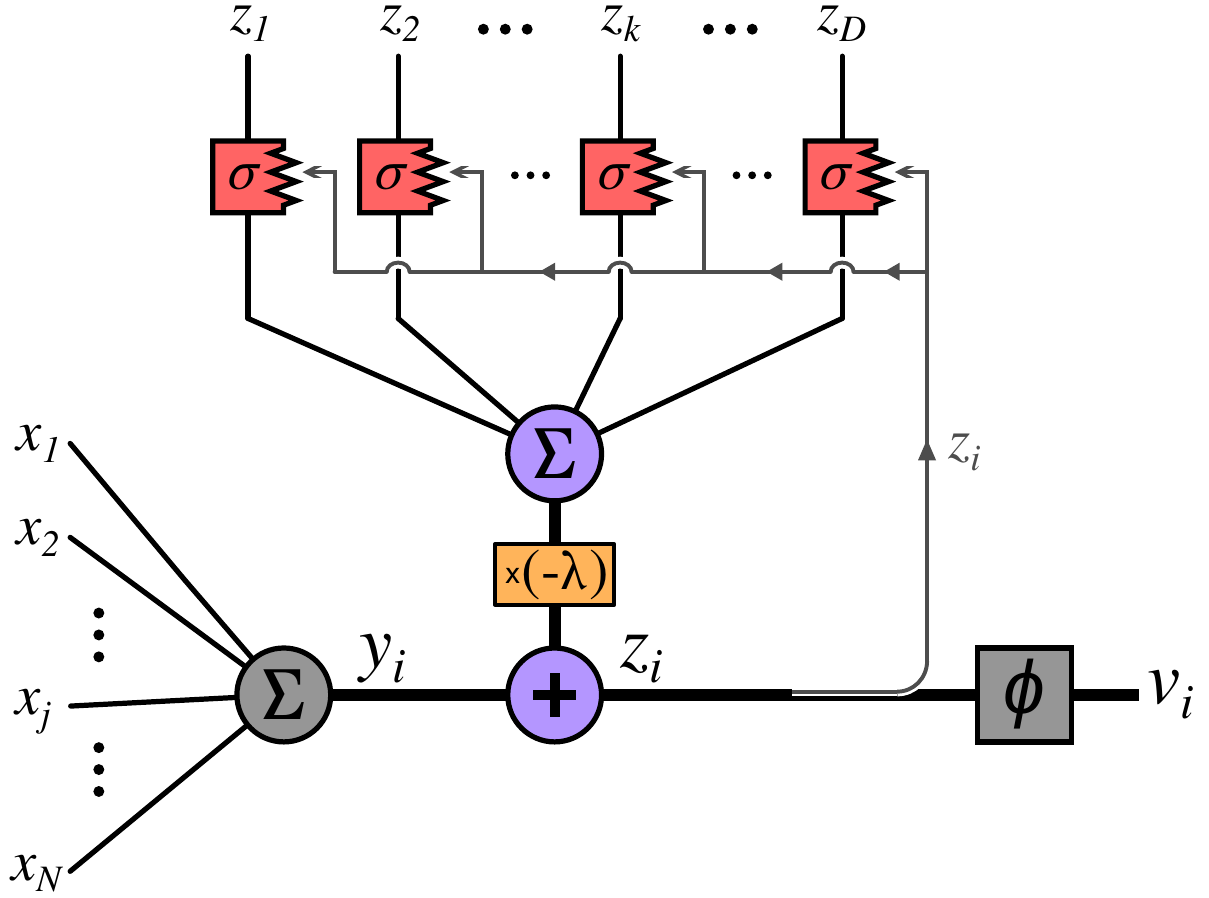}
        \caption{}
        \label{fig:schemeIBNN}
    \end{subfigure}
    \vspace{+0.1cm}
    \caption{
Schematic neuron models. Left: standard neuron model in ANNs, following a simplistic abstraction from the neuroscience of the 1950s \cite{McCulloch1943,Hubel1959} that assumed linear dendrites and an output nonlinearity. Right: proposed neuron model, based on a recent model of cortical cells \cite{Rentzeperis2025}; notice how this new model extends the standard neuron model by considering internal nonlinearities $\sigma$ (that represent dendritic processes), and these nonlinearities are affected in a dynamic way by the activity value $z_i$ (that represents the bAP going from the soma to the dendrites), itself depending on the inputs $x_j$ and the contributions $z_k$ from other neurons.
    }
    \label{fig:scheme}
\end{figure}

Here we propose a new neuron model for \acp{ANN} that extends the standard model with the addition of a new bias term,
 inspired by a very recent neuroscience work \cite{Rentzeperis2025},
that emulates the contribution of dendritic nonlinearities mediated by their interaction with \acp{BAP}.
In our formulation, the output $v_i$ of a neural unit given an input $\ve{x} \in \R[N]$ is

\begin{align}
    \label{eq:v}
    v_i &= \phi(z_i),\\
    \label{eq:z}
    z_i &= \overbrace{\sum_{j=1}^N m_{ij} x_j - b_i}^{y_i}
    \,\mLbin\, \overbrace{\lambda \sum_{k=1}^D w_{ik} \, \sigma\!\left( z_k \!-\! z_i\right)}^{\text{new bias}} , 
\end{align}

\noindent where $\lambda \in \R$, $\ve{w}_i=(w_{ik})_{k=1}^{D} \in \R[D]$ is a linear filter, and $\sigma$ is a nonlinear function.

The first term in Eq.~\eqref{eq:z} represents the response to the input $\ve{x}$ produced by the linear dendrites; notice how this term is identical to $y_i$ in Eq.~\eqref{eq:y}, which is to be expected given that the \ac{SM} only considers dendrites that are passive.
The second term in Eq.~\eqref{eq:z} models the effect of nonlinear dendrites. In this new bias term, each of the dendrites has the nonlinear response function $\sigma$, which receives two inputs: the contribution $z_k$ from a neighboring unit in the same neural layer, and the value $z_i$, which emulates the \ac{BAP} from the soma.
From Eqs.~\eqref{eq:y} and~\eqref{eq:z} it is trivial to see that, when $\lambda=0$, our proposed neuron model is identical to the standard one; see Fig.~\ref{fig:scheme} for an schematic of both models.
Importantly, if we take $\lambda$ and $\ve{w}_i$ to be fixed and predefined, then our proposed model has exactly the same number of trainable parameters as in the \ac{SM} case, i.e. $\ve{m}_i$ and $b_i$.

Given that the new bias term involves $z_i$, the definition of $z_i$ depends on $z_i$ itself and therefore Eq.~\eqref{eq:z} is {\it implicit}. For this reason, we say that an ANN that uses our model is an
\acfi{IBNN} \acused{IBNN}.
We can see from Eqs.~\eqref{eq:v} and~\eqref{eq:z} that computing the output of an \ac{IBNN} requires solving a system of coupled implicit equations, and while this type of problem does not necessarily have a solution, in our case 
when $\sigma$ is a sigmoidal function and $\lambda < T$ for some positive threshold $T$,
the solution always exists, is unique, and minimizes an energy, and consequently it can be computed with a simple gradient descent algorithm (see
\textrm{Methods}).

As an illustration, we consider a toy example in 2D where we want to find the boundary between two classes.
We train different ANNs with analogous architecture to solve this task (see
\textrm{Methods}):
a \ac{SM} network, and several \acp{IBNN} with a variety of values for the weight $\lambda$ of the implicit bias term and the maximum slope of the nonlinearity $\sigma$, denoted by $p$.
Fig.~\ref{fig:visualization_2d_lambda_p} shows the two classes in yellow and blue, and the classification boundaries obtained with the different ANNs.
As we can see, \ac{IBNN} is more expressive than the \ac{SM} network, being able to approximate well the non-convex shape of the classification boundary while the latter is only capable of producing a convex polygon.
Also, the resulting boundary in \ac{IBNN} is able to become more curved as the magnitude of the parameter $\lambda$ increases (Fig.~\ref{fig:visualization_2d_lambda_p}, left), and 
$p$ controls the curvature of the segments or arcs composing the classification boundary (Fig.~\ref{fig:visualization_2d_lambda_p}, right).

\begin{figure}[h!]
    \centering
    \begin{subfigure}[t]{0.49\textwidth}
        \centering
        \includegraphics[width=1.0\linewidth]{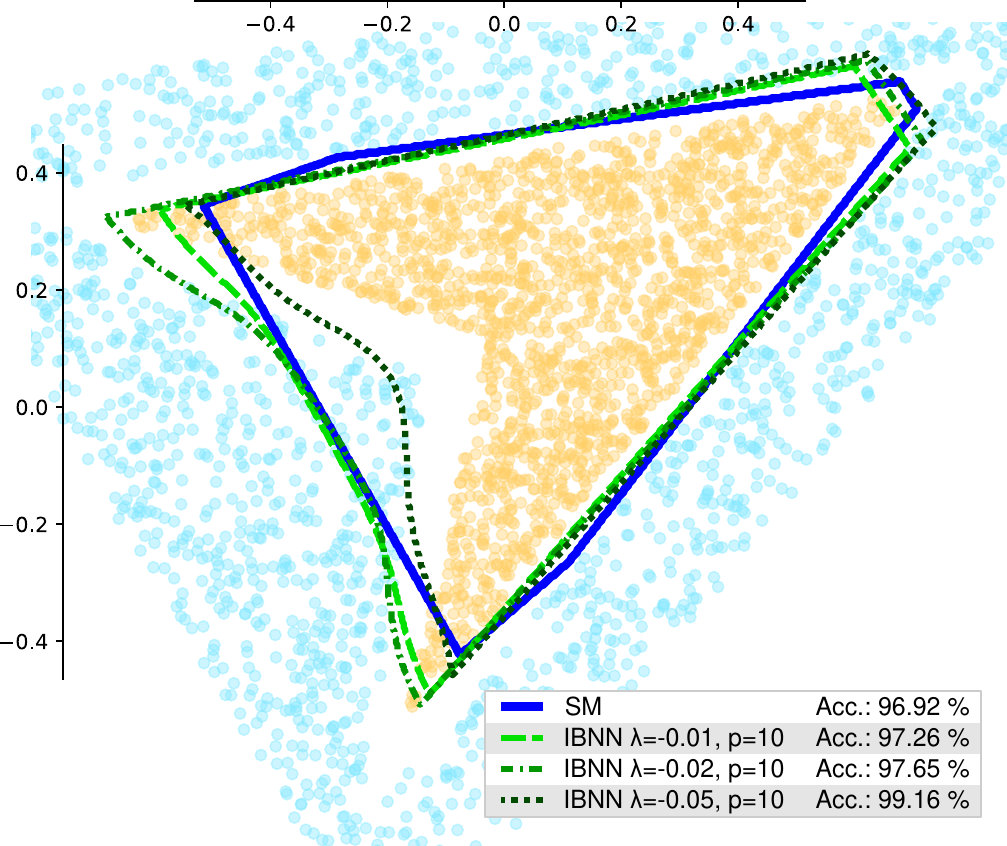}
        \caption{}
        \label{fig:visualization_2d_lambda}
    \end{subfigure}
    \hfill
    \begin{subfigure}[t]{0.49\textwidth}
        \centering
        \includegraphics[width=1.0\linewidth]{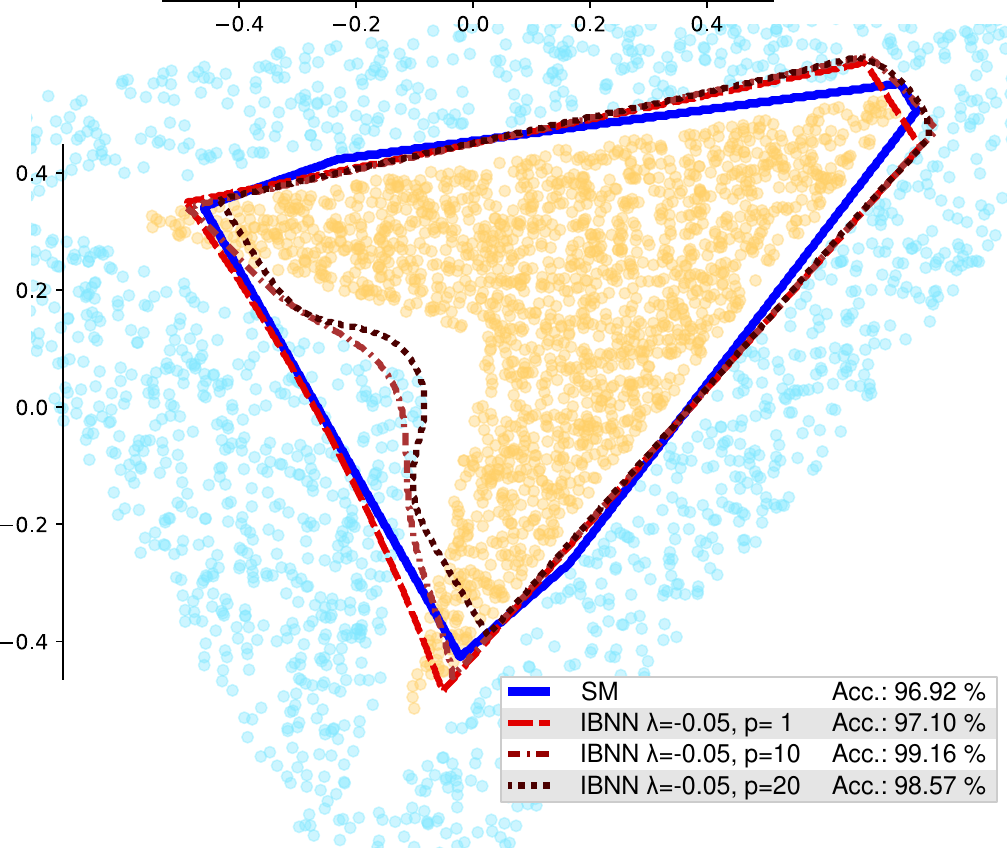}
        \caption{}
        \label{fig:visualization_2d_p}
    \end{subfigure}
    \vspace{+0.1cm}
    \caption{
        Plot of the decision boundary for \acp{ANN} with a single hidden layer of three neurons
        in a binary classification of 2D input data:
        comparison between \acp{ANN} whose hidden layer is composed of neurons according to the \ac{SM}(blue solid line in both left and right panels) and of neurons according to the bio-inspired neuron model with different weights $\lambda$ for the implicit bias term (left panel, in shades of green) and different values of the maximum slope $p$ of the sigmoidal nonlinearity $\sigma$ (right panel, shades of red).
        The resulting validation accuracy of each \ac{ANN} is included in the legend.
        Both panels suggest that the \ac{IBNN} is more expressive than the \ac{SM} network, as the \ac{IBNN} is able to approximate well the non-convex shape of the classification boundary while the \ac{ANN} using the \ac{SM} neuron is only capable of producing a convex polygonal shape. The \ac{IBNN} boundary can better fit curved boundaries as the magnitude of the parameter $\lambda$ increases (left)
        while he maximum slope $p$ of the nonlinearity $\sigma$ controls the sharpness of the transitions between the segments or arcs composing the classification boundary (right).
    }
    \label{fig:visualization_2d_lambda_p}
\end{figure}

\subsection*{Advantages of the networks that use the updated neuron model}

We perform theoretical analyses and numerical experiments comparing in each case an ANN that uses \ac{SM} neurons with an \ac{IBNN}
that has the same architecture and number of parameters but whose neurons follow the proposed model, 
thus ensuring that the benefits of \ac{IBNN} do not come from an increase in model complexity.
Numerical tests consist of solving image classification tasks over different image datasets and with different architectures,
in order to highlight the generality of the approach.

\subsubsection*{Less training data required}
Fig.~\ref{fig:partial_training_data} shows, for three different image datasets, that \ac{IBNN} consistently requires 
just a fraction ($20-82\%$) of the training data in order to achieve the best performance level of an \ac{ANN} with neurons that follow the \ac{SM}.

\begin{figure}[h!]
    \centering
    \begin{subfigure}[t]{0.32\textwidth}
        \centering
        \includegraphics[width=1.0\linewidth]{
            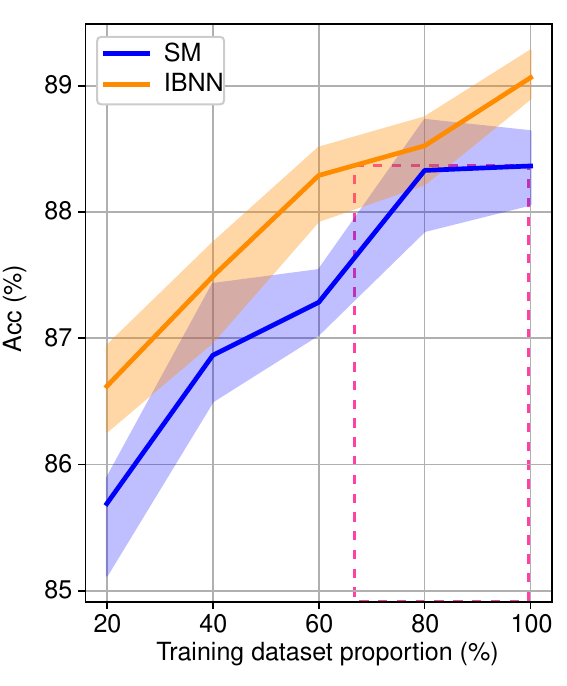
        }
        \caption{Fashion-MNIST}
    \end{subfigure}
    \hfill
    \begin{subfigure}[t]{0.32\textwidth}
        \centering
        \includegraphics[width=1.0\linewidth]{
            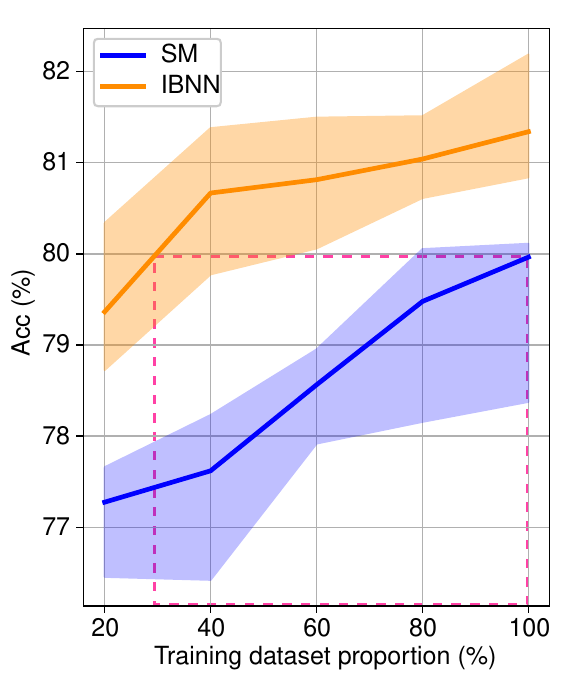
        }
        \caption{SVHN}
    \end{subfigure}
    \hfill
    \begin{subfigure}[t]{0.32\textwidth}
        \centering
        \includegraphics[width=1.0\linewidth]{
            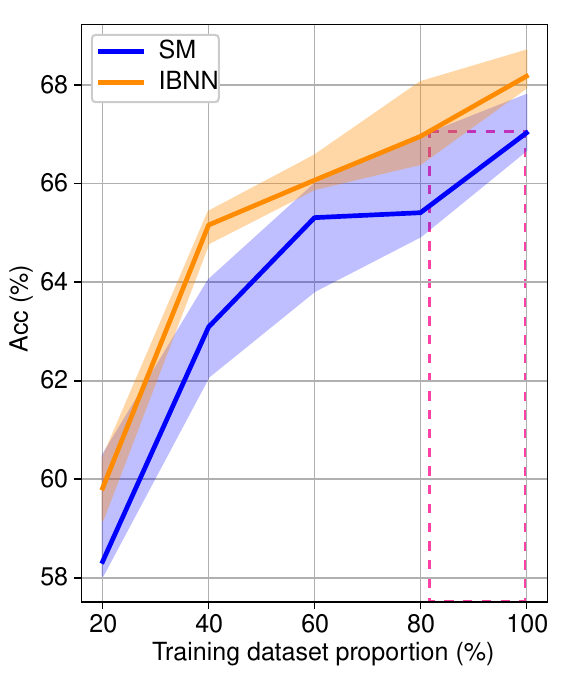
        }
        \caption{CIFAR-10}
    \end{subfigure}
    \vspace{+0.0cm}
    \caption{
        \ac{IBNN} consistently requires  just a fraction ($29-82\%$) of the training data in order to achieve the best performance level of an \ac{ANN} with neurons that follow the \ac{SM}.
        The pink dashed line indicates the best validation accuracy for \ac{SM} and the proportion of the complete dataset needed to achieve it, and the proportion for the \ac{IBNN} to achieve the same accuracy.
        Datasets are indicated under each graph and corresponding architectures indicated in \textrm{Methods}.
        Shadowed, CI of $90$\%.
    }
    \label{fig:partial_training_data}
\end{figure}

\subsubsection*{Increased robustness to input perturbations}
A common way to analyze the sensitivity of an \ac{ANN} to input perturbations is to compute the norm of the Jacobian matrix of the network with respect to the input, with smaller norm values corresponding to more robust networks \cite{Jakubovitz2018}. Under a mild set of assumptions, it can be shown that the norm of the Jacobian of an \ac{IBNN} is smaller than that of its corresponding \ac{SM} network (i.e. the network obtained by setting $\lambda=0$), and therefore the \ac{IBNN} is more robust (see \textrm{Methods}). It is notable that the formal mathematical proof presented in~\cite{Bastounis2021} to demonstrate that \ac{SM} networks are either inaccurate or unstable (if accurate) does not directly apply to the \ac{IBNN} formulation, whose structure does not satisfy the assumptions in~\cite{Bastounis2021} due to its
implicit and highly nonlinear formulation (see \textrm{Methods}).

Recently, it has become very relevant to evaluate the robustness of networks against stealth attacks \cite{Tyukin2024}, defined as a scenario where an attacker modifies one or more neurons in an \ac{ANN} in such a way that the following two conditions are met: (1) for inputs belonging to the validation set of the \ac{ANN}, the responses are as expected, and therefore no attack is observed; (2) however, for some trigger input known only to the attacker, the ANN provides the result that the attacker wants. There is proof that a one-neuron stealth attack is possible on a \ac{SM} network \cite{Tyukin2024} but we demonstrate that it would be unfeasible on an \ac{IBNN} (see \textrm{Methods}).

As experimental demonstration of the robustness capabilities that the proposed neuron model provides, 
Figs. \ref{fig:robustness_pgd} and~\ref{fig:robustness_pixle} show 
comparisons of the performance of \ac{IBNN} and \ac{SM} networks with the same number of parameters against adversarial attacks of different nature.
As we can see, \ac{IBNN} is consistently less sensitive to input perturbations.

\begin{figure}[h!]
    \centering
    \begin{subfigure}[t]{0.32\textwidth}
        \centering
        \includegraphics[width=1.0\linewidth]{
            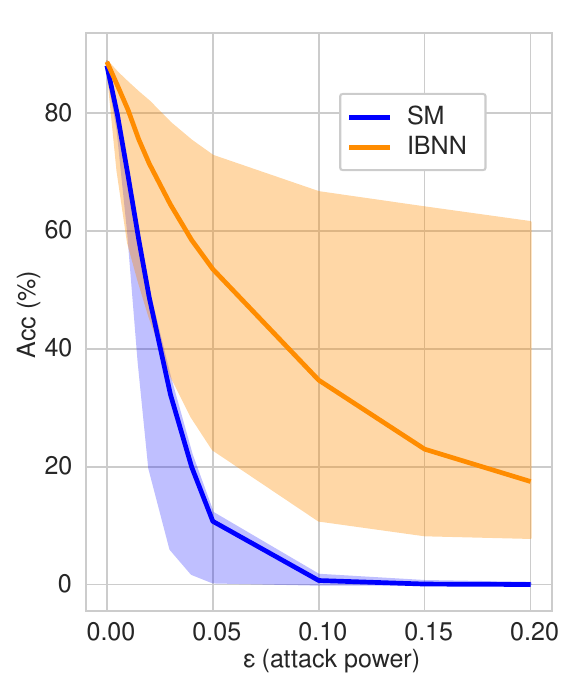
        }
        \caption{Fashion-MNIST}
    \end{subfigure}
    \hfill
    \begin{subfigure}[t]{0.32\textwidth}
        \centering
        \includegraphics[width=1.0\linewidth]{
            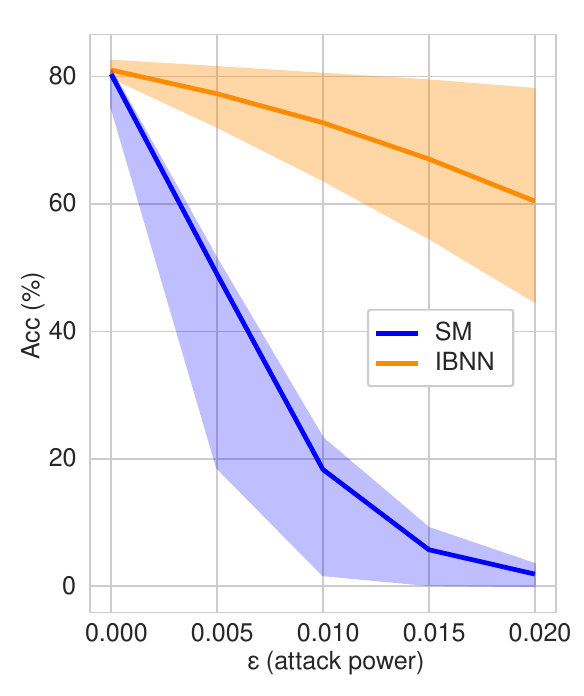
        }        \caption{SVHN}
    \end{subfigure}
    \hfill
    \begin{subfigure}[t]{0.32\textwidth}
        \centering
        \includegraphics[width=1.0\linewidth]{
            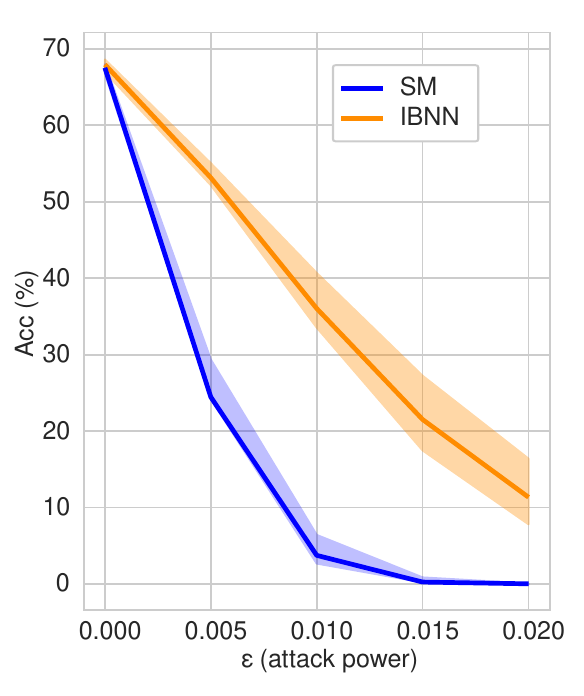
        }        \caption{CIFAR-10}
    \end{subfigure}
    \vspace{+0.0cm}
    \caption{
        \ac{IBNN} is more robust to white-box adversarial attacks.
        The curves represent the accuracy achieved by each model for different intensities of an adversarial attack, in particular \ac{PGD}~\cite{MadryPGD2018} attack (power denoted in the horizontal axis: note that a power of $0$ corresponds to the clean accuracy of the model).
        Datasets are indicated under each graph and corresponding architectures and details on the adversarial attacks are indicated in \textrm{Methods}.
        Shadowed, full range of data (CI of $100$\%).
    }
    \label{fig:robustness_pgd}
\end{figure}

\begin{figure}[h!]
    \centering
    \hfill
    \begin{subfigure}[t]{0.25\textwidth}
        \centering
        \includegraphics[width=1.0\linewidth]{
            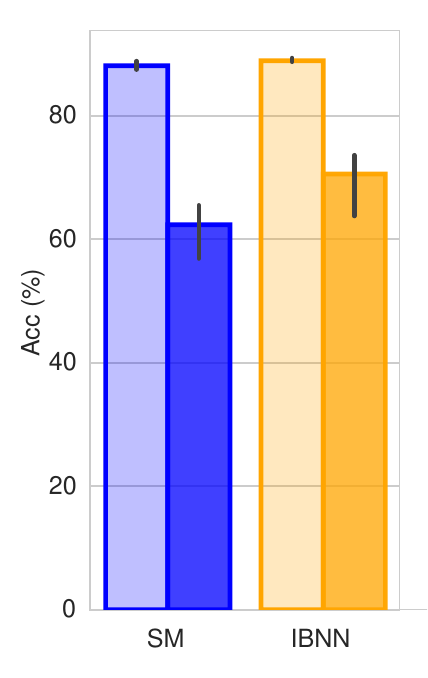
        }
        \caption{Fashion-MNIST}
    \end{subfigure}
    \hfill
    \begin{subfigure}[t]{0.25\textwidth}
        \centering
        \includegraphics[width=1.0\linewidth]{
            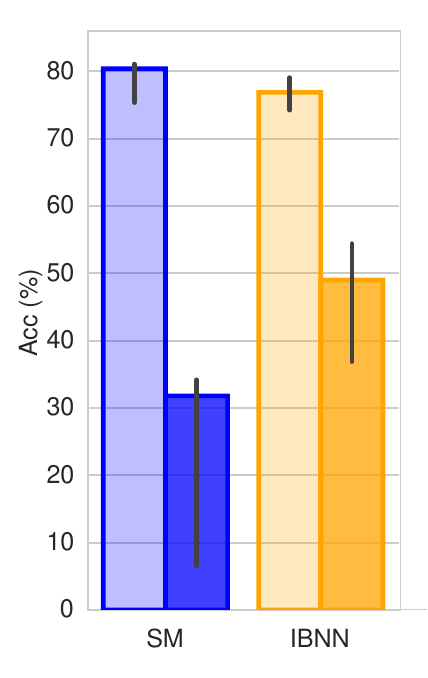
        }
        \caption{SVHN}
    \end{subfigure}
    \hfill
    \begin{subfigure}[t]{0.25\textwidth}
        \centering
        \includegraphics[width=1.0\linewidth]{
            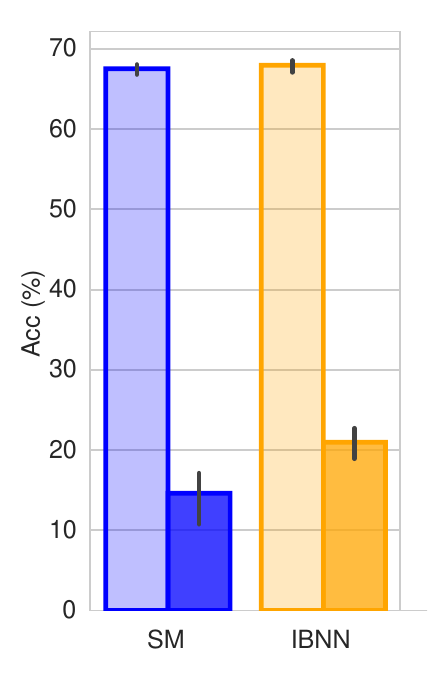
        }
        \caption{CIFAR-10}
    \end{subfigure}
    \hfill
    \vspace{+0.0cm}
    \caption{
        \ac{IBNN} is more robust to black-box adversarial attacks.
        The translucent bar represents clean accuracy for each model, whereas the solid bar represents the accuracy under an adversarial attack, in particular Pixle~\cite{PomponiPixle2022} of fixed strength.
        Datasets are indicated under each graph and corresponding architectures and details on the adversarial attacks are indicated in \textrm{Methods}.
        The black lines indicate the full range of obtained accuracies (CI of $100$\%).
    }
    \label{fig:robustness_pixle}
\end{figure}

\subsubsection*{Greater expressivity}
Assessing the expressivity of a neural network entails determining the functions it can compute.
In analogy to the classical works that demonstrate the universal approximation power of \acp{ANN} based on the \ac{SM} \cite{Hornik1989,Cybenko1989}, we establish that a single-layer \ac{IBNN} is also able to approximate any given mapping with an arbitrarily small error (see \textrm{Methods}).
But we also show that \acp{IBNN} are in fact more expressive than \ac{SM} networks, because they need a smaller number of parameters to approximate the same function (see \textrm{Methods}).

As an example of the improvement in expressivity that the updated neuron model can provide to an ANN, Fig.~\ref{fig:expressivity_all_datasets} shows that the performance of an \ac{IBNN} is consistently better than that of a \ac{SM} network with the same number of trainable parameters or, equivalently, that \ac{IBNN} 
can achieve the same performance as a \ac{SM} network with less parameters. Note that the potential consideration of trainable $\lambda$ would introduce an additional number of trainable parameters equal to the number of \ac{IBNN} layers in the network, which would make a negligible difference to the comparison shown in the figure.

\begin{figure}[h!]
    \centering
    \begin{subfigure}[t]{0.32\textwidth}
        \centering
        \includegraphics[width=1.0\linewidth]{
            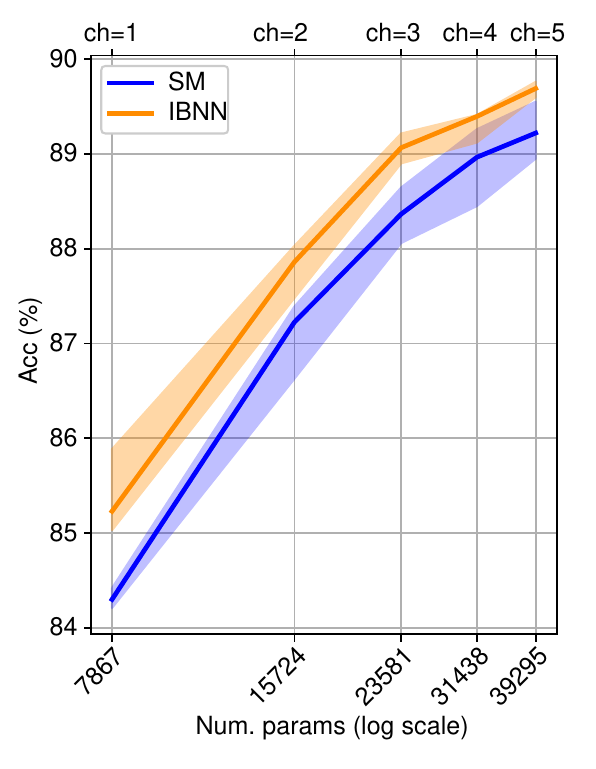
        }
        \caption{Fashion-MNIST}
    \end{subfigure}
    \hfill
    \begin{subfigure}[t]{0.32\textwidth}
        \centering
        \includegraphics[width=1.0\linewidth]{
            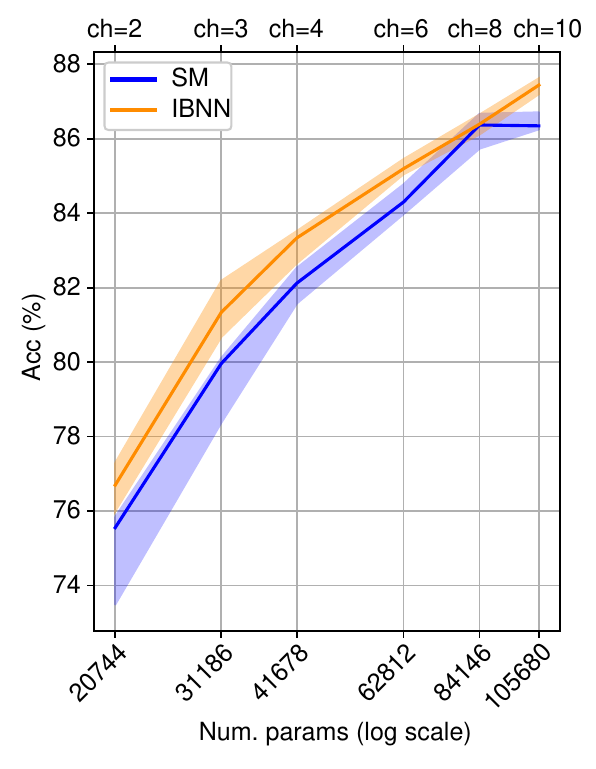
        }
        \caption{SVHN}
    \end{subfigure}
    \hfill
    \begin{subfigure}[t]{0.32\textwidth}
        \centering
        \includegraphics[width=1.0\linewidth]{
            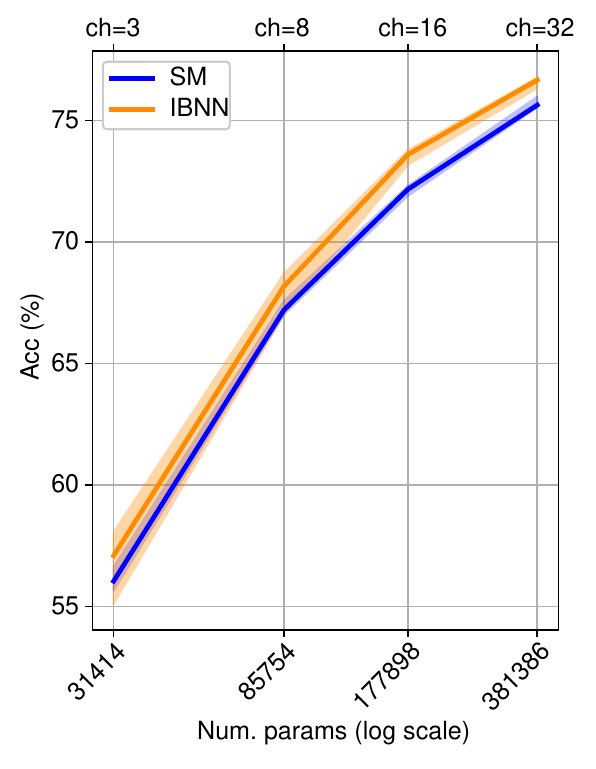
        }
        \caption{CIFAR-10}
    \end{subfigure}
    \vspace{+0.0cm}
    \caption{
        \ac{IBNN} is more expressive: it achieves higher accuracy with the same number of trainable parameters or, equivalently, 
        it requires less parameters to perform as well as a \ac{SM} network. Variation in the number for parameters caused by the increase in the number of channels of the layers of the backbone is accounted for in each graph. 
        Datasets are indicated under each graph and corresponding architectures are indicated in \textrm{Methods}.
        Shadowed, CI of $90$\%.         
    }
    \label{fig:expressivity_all_datasets}
\end{figure}

\subsubsection*{Faster learning}
Fig.~\ref{fig:learning_speed_all_datasets} shows
that \ac{IBNN} requires less learning steps to achieve 
the same classification performance as a \ac{SM} network: in all cases, this equated to less than half the number of epochs.

\begin{figure}[h!]
    \centering
    \begin{subfigure}[t]{0.32\textwidth}
        \centering
        \includegraphics[width=1.0\linewidth]{
            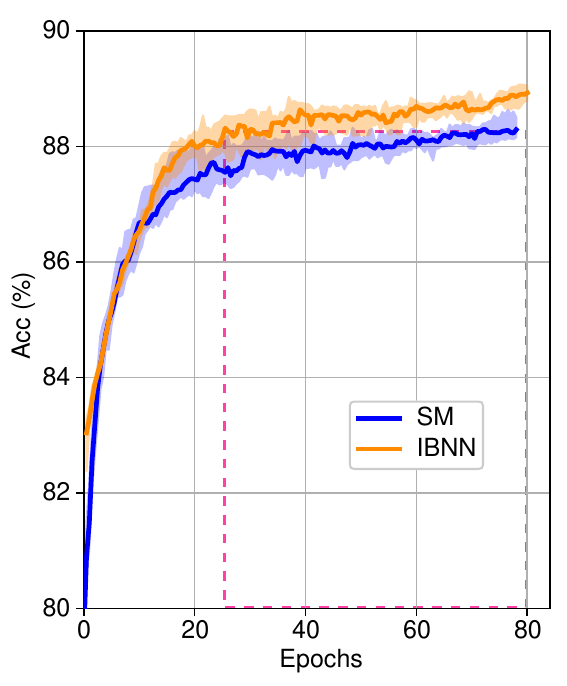
        }
        \caption{Fashion-MNIST}
    \end{subfigure}
    \hfill
    \begin{subfigure}[t]{0.32\textwidth}
        \centering
        \includegraphics[width=1.0\linewidth]{
            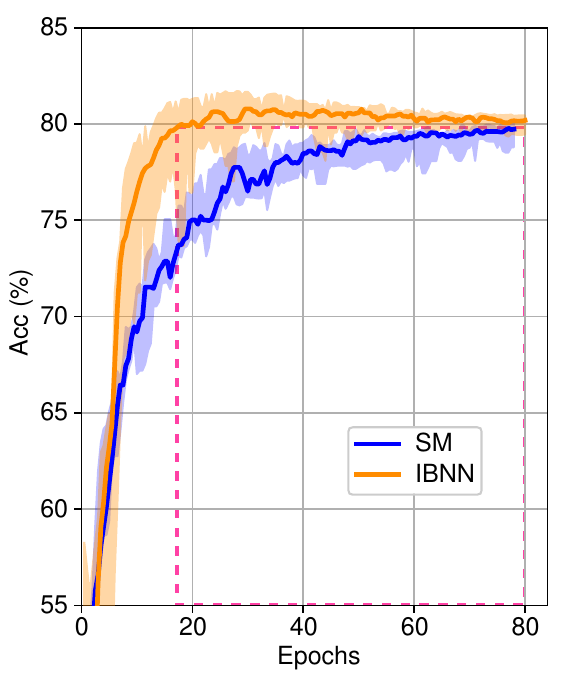
        }
        \caption{SVHN}
    \end{subfigure}
    \hfill
    \begin{subfigure}[t]{0.32\textwidth}
        \centering
        \includegraphics[width=1.0\linewidth]{
            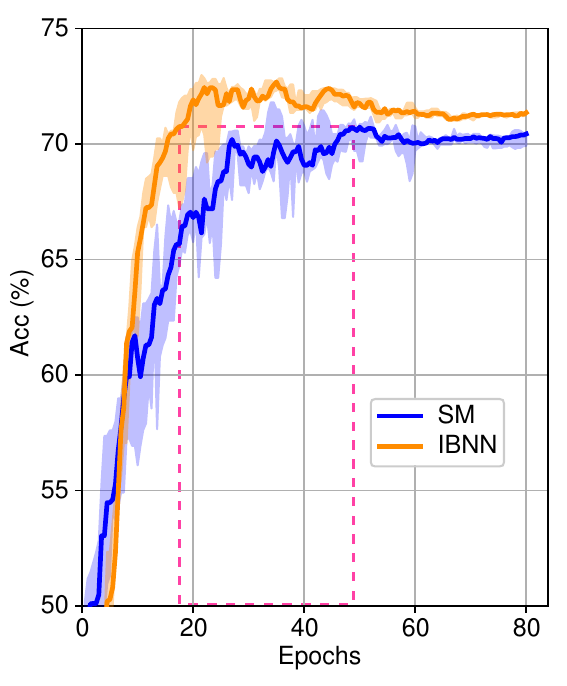
        }
        \caption{CIFAR-10}
    \end{subfigure}
    \vspace{+0.0cm}
    \caption{
        \ac{IBNN} learns faster: \ac{IBNN} matches in $18$-$25$ epochs the best validation accuracy reached by the \ac{SM} network within a $80$-epoch training, which, in all cases, happened after the $53$-th epoch. 
        The pink dashed line indicates the best validation accuracy for the \ac{SM} over the considered training epochs and epoch when it was achieved, and the epoch for the \ac{IBNN} to achieve the same accuracy.
        Datasets are indicated under each graph and corresponding architectures are indicated in \textrm{Methods}.
        Shadowed, CI of $90$\%.         
    }
    \label{fig:learning_speed_all_datasets}
\end{figure}

\subsubsection*{Less memorization}
When networks have a very large number of parameters they can manage to learn imperfect and even mislabeled training data, in a process called memorization \cite{FeldmanZhang2020}.
This is clearly problematic, as the performance of the network may become brittle. 
The occurrence of memorization correlates with the curvature of the loss function with respect to the input \cite{Garg2024}, and we
provide a mathematical analysis suggesting 
that \ac{IBNN} has a smaller curvature and therefore memorizes less than a \ac{SM} network (see \textrm{Methods}).

Fig.~\ref{fig:memorization_all_datasets} shows 
an experiment where the training data has been corrupted with varying percentages of mislabeled samples (details in \textrm{Methods}): 
the advantage of the best validation accuracy achieved by \ac{IBNN} over the best validation achieved by an analogous \ac{SM} network is consistently
larger when the training set includes data with wrong labels,
indicating that \ac{IBNN} memorizes less.


\begin{figure}[h!]
    \hfill
    \begin{subfigure}[t]{0.32\textwidth}
        \centering
        \includegraphics[width=1.0\linewidth]{
            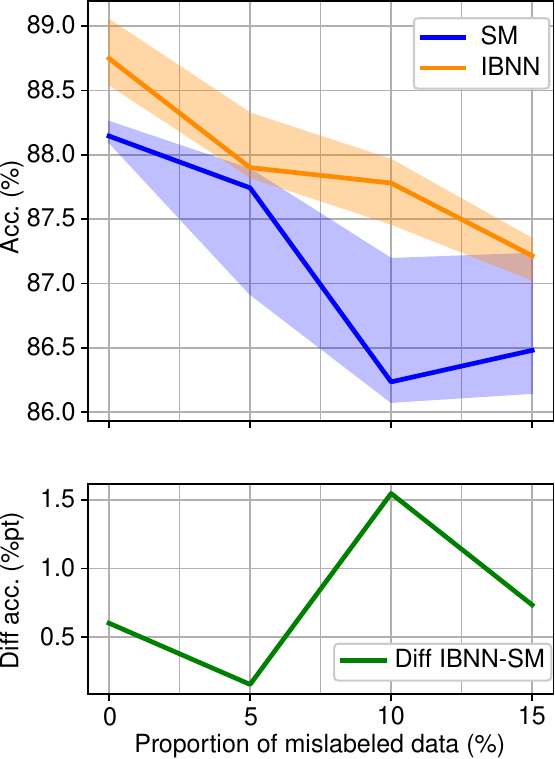
        }
        \caption{Fashion-MNIST}
    \end{subfigure}
    \hfill
    \begin{subfigure}[t]{0.32\textwidth}
        \centering
        \includegraphics[width=1.0\linewidth]{
            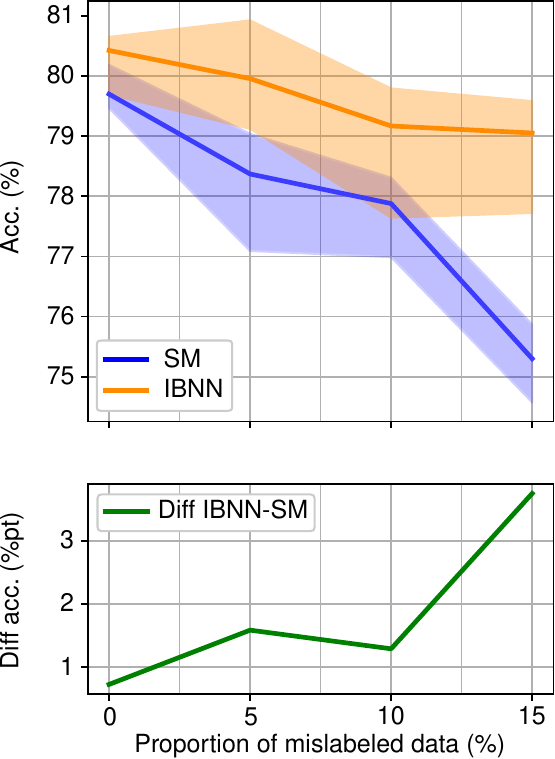
        }
        \caption{SVHN}
    \end{subfigure}
    \hfill
    \begin{subfigure}[t]{0.32\textwidth}
        \centering
        \includegraphics[width=1.0\linewidth]{
            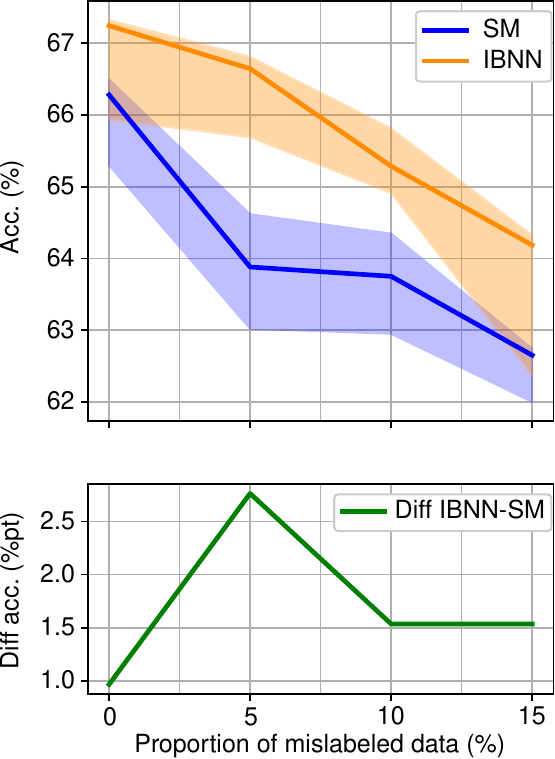
        }
        \caption{CIFAR-10}
    \end{subfigure}
    \hfill
    \vspace{+0.1cm}
    \caption{
        \ac{IBNN} memorizes less: the improvement in validation accuracy of \ac{IBNN} over the \ac{SM} network is constently larger when the training set includes data with wrong labels
        and, in certain instances, the difference even widens as the proportion of wrong labels increases (see panel (b)).
        Datasets are indicated under each graph and corresponding architectures are indicated in \textrm{Methods}.
        For each dataset: top row, best validation accuracy achieved for each model for different percentage of wrong training labels;
        bottom row, difference of the validation of \ac{IBNN} and \ac{SM} depicted in the top row.
        Shadowed, CI of $90$\%.         
    }
    \label{fig:memorization_all_datasets}
\end{figure}

\section*{Discussion} 

While ANNs continue to serve as the core of many state-of-the-art AI algorithms, they still have significant limitations that have not been overcome despite extensive research efforts.
There appears to be an emerging consensus that the solution to these problems will not arrive simply from scaling up current ANNs, but that fundamentally new approaches are needed instead \cite{Zador2023}. 
In this article we take inspiration from neuroscience and propose to use a new, compact and very effective model of cortical cells \cite{Rentzeperis2025} to replace the omnipresent but clearly outdated point neuron model of ANNs.
The novelty of the proposed approach lies in the combination of several bio-inspired elements, including internal dendritic-like nonlinearities that become dynamic and input-dependent through the modeling of the bAP, and lateral interactions among neurons that result in a system of coupled and implicit differential equations. 
In the NeuroAI literature there are recent works that introduce more involved models of neural units \cite{Bertalmio2020Evidence,Paiton2020,Rawat2024,Chavlis2025}, 
but they fail to emulate the bAP, or assume linear dendrites, or disregard lateral interactions.
And, while some AI approaches use implicit layers and differential equations \cite{BaiKolter2019}, they lack an equivalent to the bAP that might make the nonlinearities input-dependent.

The experimental and theoretical results presented in the article support our original hypothesis, namely, that the overly simplistic neuron model of ANNs contributes to their fundamental limitations, and that an updated and biologically inspired neuron model, one that takes into account the complexity of dendritic processing, may allow ANNs to achieve significant advantages. 
We have shown that ANNs using the new model demonstrate gains over a number of aspects that include expressivity, robustness and learning, 
and this aligns with the fact that nonlinear dendrites bring these same advantages to cortical computations and neural behaviors \cite{Poirazi2020}.

The experiments have focused on image data, and the model comes from visual neuroscience, but we expect the approach to generalize well, making extensions to other data modalities simple, in the same way that deep learning is applied to any sort of data despite its being based on the hierarchical model that Hubel and Wiesel proposed for neurons in the visual cortex \cite{Hubel1959}.
The \ac{IBNN} parameters $\{\ve{m}_i,b_i\}$ in all experiments have been learned, but we have also tested a second mode of operation for our formulation where these parameters are simply taken from a \ac{SM} network that solves the same task and has already been trained; the original mode produces more accurate results and is better suited for scenarios where training data is scarce (given that \ac{IBNN} needs less training data), whereas the second mode can be seen as a fast and simple method to take an existing AI system and make it more robust to input perturbations, without the need to train it again
(\textrm{Supplementary Information}). 

There are several research questions related to the proposed model that are worth exploring as future work.
To begin, 
we must note that in our experiments the optimum value for the hyperparameter $\lambda$ was either found by grid search 
or through pre-training (\ie by optimizing for accuracy); moving forward we believe that $\lambda$ should always be considered a trainable parameter.
In the same vein, the spatial extent of the filters $\ve{w}_i$ and the maximum slope $p$ of the nonlinearity $\sigma$ should also be learned.
This would produce just a marginal increase in the number of parameters while it could enhance the expressivity of the network even further, although the mathematical analysis could become much more involved.
Related to this, the use of values of $\lambda$ above a certain positive threshold might imply the existence of multiple solutions for a single input, which could allow to represent with higher fidelity the modeling of biological processes through partial differential equations, where the final state of the process is often path-dependent \cite{Khona2022}.
We will also develop numerical techniques specifically tailored to \acp{IBNN}, in order to reduce the computational complexity associated to the iterative nature of finding the solution to the implicit system of equations, and also to the presence of multiple cross-comparisons in the nonlinear term of the new neuron model.

Secondly, while the proposed model was inspired by the V1 model in \cite{Rentzeperis2025}, for the sake of simplicity we did not emulate it fully, omitting one term of their formulation. We will incorporate that term to our model, which will introduce a non-monotonic internal nonlinearity, and analyze in which contexts this may provide an advantage (a recent work suggests that the use of non-monotonic nonlinearities might be key to predict time-varying phenomena \cite{Luna2025}).

Thirdly, the second term of our model (\ie the new bias in Eq.~\eqref{eq:v}) leads to the observation that \ac{IBNN} has the ability to push for independent representations, because if the activities $z_i$ and $z_j$ are very similar 
then the neural units $i$ and $j$ do not influence each other; this is consistent with the fact that the V1 model that motivated our approach has a direct connection with the sparse coding framework \cite{Rentzeperis2025}.
We view this property of our proposed neuron model as providing a complementary behavior to the attention mechanism employed by Transformer architectures \cite{Vaswani2017},
it would be interesting to test if their combination (\ie a Transformer using the updated neuron model) provides substantial gains in expressivity.

Finally, the mathematical form of our model, where the solution minimizes an energy and
can be found through an iterative algorithm analogous to the numerical approximation of a differential equation (\textrm{Methods}, Eq.~\eqref{eq:iterative_calculation_of_IBNN}),
links our approach with energy models for AI \cite{Hopfield1984} and also with causal modeling for machine learning \cite{Scholkopf2022}, a topic of the utmost importance as it has been argued that 
the major unsolved problems in AI are fundamentally tied to causality \cite{Pearl2018}.

\begin{methods_section}

    \section*{Methods} 
    \label{sec:methods}
        \subsection*{Model solution}
        \label{sec:methods:model_solution}
        %
        \input{_methods__model_solution}
       \subsection*{\ac{IBNN} in matrix notation}
        \label{sec:methods:ibnn_in_matrix_notation}
        %
        \input{_methods__ibnn_matrix_notation}
        \subsection*{Expressivity}
        \label{sec:methods:expressivity}
        %
        \input{_methods__expressivity}

        \subsection*{Comparative analysis of the robustness of \ac{IBNN}}
        \label{sec:methods:comparative_analysis_of_adversarial_robustness}
        %
        \input{_methods__comparative_analysis_of_robustness}

        \subsection*{Resistance to memorization}
        \label{sec:methods:memorization}
        %
        \input{_methods__memorization}
        \subsection*{Implementation of the \ac{IBNN} layers}
        \label{sec:methods:implementation_of_IBNN}
        %
        \input{_methods__implementation_of_IBNN}
        \subsection*{Experimental details}
        \label{sec:methods:experimental_details}
        \input{_methods__experimental_details}

\end{methods_section}


\begin{code_availability_section}

    \section*{Code availability}
    
    The code used to carry out the simulations and analysis is available at \url{https://github.com/vmg-io-csic/ibnn}.

\end{code_availability_section}


\section*{Author contributions}

Conceptualization and methodology: M.B.;
mathematical analysis: R.M., T.B., M.B., S.L.;
numerical implementation: R.M., E.V.S.;
experiments: R.M., E.V.S., R.S.M., J.H.M.;
analysis of results: all authors;
writing: R.M., M.B., T.B., S.L.


\begin{acknowledgements_section}

    \section*{Acknowledgments}
    
    This work has been supported by: Programa Fundamentos (ref.~VIS4NN) of Fundaci{\'o}n BBVA 2022 (Spain); Programa Tecnolog{\'\i}as (ref. TEC~2024/COM-322) of Comunidad de Madrid (Spain); Proyectos de Generaci{\'o}n de Conocimiento 2024 (ref. PID2024-161841NB-I00) of Ministerio de Ciencia, Innovaci{\'o}n y Universidades (Spain); and Programa Momentum (ref. MMT24-IO-02) of Consejo Superior de Investigaciones Cient{\'\i}ficas (Spain), funded by European Commission–NextGenerationEU through the Spanish Recovery, Transformation and Resilience Plan.

\end{acknowledgements_section}


\section*{Additional information}
\subsection*{Supplementary information}

\begin{pointer_to_suppl_info_submission}
    The present manuscript refers to \textrm{Supplementary Information} contained within a separate document corresponding to the same submission.
\end{pointer_to_suppl_info_submission}

\begin{pointer_to_suppl_info_arxiv}
    The present manuscript refers to \textrm{Supplementary Information} included as an appendix.
\end{pointer_to_suppl_info_arxiv}



\bibliography{bibIBNN}


\begin{supplementary_information_section}

    \newpage
    
    \renewcommand{\appendixpagename}{Supplementary Information}
    \begin{appendices} 
    \appendixpage 
    
        \section*{Creating an IBNN by transferring the weights from an existing SM network already trained for the same task}
        \setcounter{equation}{0} 
        \renewcommand{\theequation}{S\arabic{equation}} 
        \input{_supplementary_information__ibnn_from_sm}

    \end{appendices}

\end{supplementary_information_section}


\end{document}

%% file: _custom_added_preamble.tex
\usepackage[utf8]{inputenc}
\usepackage{amsmath,amssymb,amsfonts,amsthm}%
\usepackage{mathtools}
\usepackage{mathrsfs}%
\usepackage{natbib}
\usepackage{caption}
\usepackage{comment}
\usepackage{subcaption}
\usepackage{tabularx}
\usepackage{bm}
\usepackage{bbm}
\usepackage{dsfont}
\usepackage{framed}
\usepackage{color}
\usepackage{breqn}
\usepackage{enumerate}
\usepackage{url}
\usepackage{soul}
\usepackage{acronym}
\usepackage{xparse}
\usepackage{nameref}
\usepackage{hyperref}
\usepackage[title]{appendix}

\usepackage{etoolbox}
\newbool{renderfigures}
\setbool{renderfigures}{false} 

\theoremstyle{thmstyleone}%

\newtheorem{proposition}{Proposition}

\newtheorem{corollary}{Corollary}[proposition]%

\theoremstyle{thmstyletwo}%

\theoremstyle{thmstylethree}%

\newcolumntype{L}[1]{>{\raggedright\arraybackslash}p{#1}}
\newcolumntype{C}[1]{>{\centering\arraybackslash}p{#1}}
\newcolumntype{R}[1]{>{\raggedleft\arraybackslash}p{#1}}
\newcommand{\eg}{e.g.~}
\newcommand{\ie}{i.e.~}

\newcommand{\etal}{~et~al.~}

\newcommand{\mLbin}{-}
\newcommand{\mLuni}{-}
\newcommand{\pLbin}{+}

\newcommand{\geqL}{\geq}

\newcommand{\leqL}{\leq}
\newcommand{\R}[1][]{\mathbb{R}^{#1}}

\DeclareMathOperator\ReLU{ReLU}
\DeclareMathOperator\SM{SM}

\DeclareMathOperator\IBNN{IBNN}

\DeclareMathOperator{\diag}{diag}

\DeclareMathOperator{\prox}{prox}

\newcommand{\ve}[1]{\mathbf{#1}}

\newcommand{\veT}[1]{\ve{#1}^T}

\newcommand{\ma}[1]{\mathrm{#1}}
\newcommand{\maT}[1]{\ma{#1}^T}

\newcommand{\id}[1]{\textrm{Id}_{{#1}\!\times\!{#1}}}
\DeclareDocumentCommand \der { d[] d()}
{
    \IfNoValueTF{#1}{
        \IfNoValueTF{#2}{\mathrm{D}}{\mathrm{D}_{|#2}}
    }{  \IfNoValueTF{#2}{\mathrm{D}{#1}}{\mathrm{D}{#1}_{|#2}}
    }
}
\newcommand{\diff}[2]{\frac{\mathrm{d} #1}{\mathrm{d} #2}}
\newcommand{\diffp}[2]{\frac{\partial #1}{\partial #2}}

\usepackage{fancybox,graphbox}
\usepackage[many]{tcolorbox}
\usepackage{etoolbox}
\definecolor{mycolor}{rgb}{0.122, 0.435, 0.698}
\makeatletter
\newcommand{\mybox}[1]{%
  \setbox0=\hbox{#1}%
  \setlength{\@tempdima}{\dimexpr\wd0+13pt}%
  \begin{tcolorbox}[colframe=mycolor,boxrule=0.5pt,arc=4pt,
      left=6pt,right=6pt,top=6pt,bottom=6pt,boxsep=0pt,width=\textwidth]
    #1
  \end{tcolorbox}
}

\acrodef{AI}[AI]{artificial intelligence}
\acrodef{AA}[AA]{adversarial attack}
\acrodef{AE}[AE]{adversarial example}
\acrodef{AN}[AN]{artificial neuron}
\acrodefplural{AN}[ANs]{artificial neurons}
\acrodef{ANN}[ANN]{artificial neural network}
\acrodef{AR}[AR]{adversarial robustness}
\acrodef{AT}[AT]{adversarial training}
\acrodefplural{AR}[AR]{adversarial robustnesses}
\acrodef{BAP}[bAP]{backpropagating action potential}
\acrodef{CNN}[CNN]{convolutional neural network}
\acrodef{DEQ}[DEQ]{Deep Equilibrium}
\acrodef{DL}[DL]{deep learning}
\acrodef{FGSM}[FGSM]{Fast Gradient Sign Method}
\acrodef{FC}[FC]{fully-connected}
\acrodef{FR}[FR]{fooling rate}
\acrodef{IBNN}[IBNN]{implicit bias neural network}
\acrodef{LCA}[LCA]{locally competitive algorithm}
\acrodef{LR}[LR]{learning rate}
\acrodef{MSE}[MSE]{Mean Squared Error}
\acrodef{NeuralODE}[Neural ODE]{Neural Ordinary Differential Equations}
\acrodef{PDE}[PDE]{partial differential equation}
\acrodef{PGD}[PGD]{Projected Gradient Descent}
\acrodef{SGD}[SGD]{Stochastic Gradient Descent} 
\acrodef{SM}[SM]{standard model}
\acrodef{ODE}[ODE]{ordinary differential equation}
\acrodef{OOD}[OOD]{out-of-distribution}
\acrodef{SVM}[SVM]{Support Vector Machine}
\acrodef{UAT}[UAT]{u\textbf{}niversal approximation theorem}
\acrodefplural{UAT}[UATs]{universal approximation theorems}
\acrodef{UCN}[UCN]{uniform convolutional network}
\acrodef{V1}[V1]{primary visual cortex}


%% file: _methods__model_solution.tex

\subsubsection*{Existence, uniqueness, and energy minimization}
\label{sec:methods:existence_uniqueness_and_energy}


We will next give some conditions on $\lambda$ for which we have (i)~existence and (ii)~uniqueness of the solutions of Eq.~\eqref{eq:z}. To that aim, we establish a connection between Eq.~\eqref{eq:z} and energy functionals of the form



\begin{equation}
    \label{eq:methods:energy_function}
    E_{\lambda}(\ve{z}; \ve{y}) = g(\ve{z}; \ve{y}) - f_{\lambda}(\ve{z}) \,, \ve{z}, \ve{y} \in \R[D]\; \textrm{with}\quad
    \left\{
        \begin{matrix*}[l]
            g(\ve{z}; \ve{y})
            & = \displaystyle
            \hphantom{\mLuni} \frac{1}{2} {\big\Vert \ve{z} - \ve{y} \big\Vert}_{2}^2
            \\
            f_{\lambda}(\ve{z})
            & = \displaystyle
             \frac{\lambda}{2}  \sum_{i=1}^D \sum_{k=1}^D w_{ik} S (z_k - z_i)
        \end{matrix*}
    \right. \, ,
\end{equation}

where

 \begin{itemize}
     \label{itemize:assumpions_for_existence_and_uniqueness}
     \item
         $S$ is a primitive function of $\sigma$, and
     \item
         $w_{ik}$ are symmetric, positive, and normalized
         (\ie $w_{ik} = w_{ki} \geq 0$ and $\sum_{k=1}^{N} w_{ik} = 1$),
         $0 < i,k \leq D$.
 \end{itemize}

\begin{proposition}
\label{prop: relation solutions equation-critical points energy}
The solutions of Eq.~\eqref{eq:z}, if they exist, are the critical points of $E_{\lambda}(\cdot;M\ve{x}-\ve{b})$ provided that $S$ is even.
\end{proposition}

\begin{proof}
Denoting by $\nabla_{\ve{z}}$ the gradient operator with respect to $\ve{z}$, we have that
$$
\nabla_{\ve{z}} g(\ve{z};\ve{y}) = \ve{z}-\ve{y}
$$
Moreover, if $S$ is even, then we have from the fact that $S'=\sigma$ and $w$ are symmetric that
$$
(\nabla_{\ve{z}} f_{\lambda}(\ve{z}))_i =  - \lambda \sum_{k=1}^D  w_{ik} \sigma (z_k - z_i),
$$
from which we deduce that 
$$
(\nabla_{\ve{z}} E_{\lambda}(\ve{z};\ve{y}))_i = 0 \Longleftrightarrow z_i = y_i - \lambda \sum_{k=1}^D  w_{ik} \sigma (z_k - z_i).
$$
\end{proof}

\begin{proposition}
\label{proposition:existence}
If $S$ is positive, even and satisfies that $S(z) \leq A z^2$, for some  $A \geq 0$ then $E_{\lambda}$ is coercive with respect to the first variable for any $\lambda <   \dfrac{1}{4A}$.
\end{proposition}

\begin{proof}
We have 
\begin{equation}
\label{eq: ineq 1}
\dfrac{1}{2} {\| \ve{z} - \ve{y} \|}^2 \geq \dfrac{1}{2} \| \ve{z} \|^2 - \|\ve{z}\| \| \ve{y}\| + \dfrac{1}{2} \| \ve{y} \|^2 
\end{equation}
from the Cauchy-Schwartz inequality.\\
\\
As $S$ is positive, the term $f_{\lambda}(\ve{z})$ in (\ref{eq:methods:energy_function}) is positive if $\lambda>0$, from which follows that 
$$
E_{\lambda}(\ve{z};\ve{y}) \geq  \dfrac{1}{2} \| \ve{z}\|^2 - \|\ve{z}\| \| \ve{y} \| + \dfrac{1}{2} \| \ve{y} \|^2 \qquad \textrm{if} \quad \lambda  < 0,
$$
which gives that $E_{\lambda}$ is coercive with respect to $\ve{z}$ if $\lambda <0$.\\
\\
On the other hand, if $\lambda>0$ and $S(z) \leq A z^2$, we have
\begin{align}
\label{eq: ineq 2}
 \dfrac{\lambda}{2} \sum_{i=1}^D \sum_{j=1}^D w_{ij} S (z_j - z_i) & \leq  \dfrac{\lambda}{2} A \sum_{i=1}^D \sum_{j=1}^D w_{ij} {(z_j - z_i)}^2  \nonumber \\
& \leq  \dfrac{\lambda}{2} A \sum_{i=1}^D \sum_{j=1}^D w_{ij} (z^2_j + z^2_i + 2 |z_j| |z_i|)   \nonumber \\
& \leq  \lambda A \| \ve{z}\|^2 + \lambda A \sum_{i=1}^D \sum_{j=1}^D w_{ij} |z_j| |z_i|
\end{align}
from Cauchy-Schwartz inequality.\\
\\
Note that we have 
$$
\sum_{i=1}^D \sum_{j=1}^D w_{ij} z_i^2 = \sum_{i=1}^D \sum_{j=1}^D w_{ij} z_j^2 = \| \ve{z}\|^2
$$
as the kernel $w$ is normalized. 
The Cauchy-Schwartz inequality applied to the vectors $\ve{a},\ve{b} \in \mathbb{R}^{D \times D}$:
$$
a_{ij} = \sqrt{w_{ij}} |z_i| \qquad b_{ij} = \sqrt{w_{ij}} |z_j|
$$
gives $\langle \ve{a},\ve{b} \rangle \leq \|\ve{a}\|\|\ve{b}\|$, i.e.
\begin{align}
\label{eq: ineq 3}
\sum_{i=1}^D \sum_{j=1}^D w_{ij} |z_i| |z_j|  & \leq \left( \sum_{i=1}^D \sum_{j=1}^D w_{ij} z^2_j \right)^{1/2} \left( \sum_{i=1}^D \sum_{j=1}^D w_{ij} z^2_i \right)^{1/2}  \\
& \leq \| \ve{z} \|^2
\end{align}
We deduce from Eqs.~\eqref{eq: ineq 1}, \eqref{eq: ineq 2} and~\eqref{eq: ineq 3} that
\begin{align*}
E_{\lambda}(\ve{z};\ve{y}) & \geq \dfrac{1}{2} \| \ve{z}\|^2 - \|\ve{z}\| \| \ve{y} \| + \dfrac{1}{2} \| \ve{y} \|^2  -  \lambda A \| \ve{z} \|^2 - \lambda A \| \ve{z} \|^2 \\
& \geq \| \ve{z} \|^2 (\dfrac{1}{2} -  2 \lambda A) - \|\ve{z}\| \| \ve{y} \| + \dfrac{1}{2} \| \ve{y} \|^2.
\end{align*}
We deduce that $E_{\lambda}$ is coercive if $\lambda <  \dfrac{1}{4A}$. 
\end{proof}

\begin{proposition}
\label{prop: strictly convexity}
If $S$ is convex and twice differentiable with second order derivative bounded above by $K>0$, then the energy in Eq.~\eqref{eq:methods:energy_function} is strictly convex  with respect to the first variable for $\lambda < \frac{1}{2K}$.
\end{proposition}
\begin{proof}
We have
\begin{align*}
\dfrac{d}{dt}E_{\lambda}(\ve{z}+t \boldsymbol{\eta};\ve{y}) = &  \dfrac{1}{2} (
2 \langle \ve{z},\boldsymbol{\eta} \rangle + 2t \| \boldsymbol{\eta}\|^2  -2 \langle \boldsymbol{\eta}, \ve{y} \rangle
)  \\
& - \dfrac{\lambda}{2}  \sum_{i=1}^D \sum_{j=1}^D w_{ij} (\eta_j-\eta_i) S'(z_j + t \eta_j - z_i - t \eta_i)
\end{align*}
and 
$$
\dfrac{d^2}{dt^2}E_{\lambda}(\ve{z}+t\boldsymbol{\eta};\ve{y}) =  \| \boldsymbol{\eta} \|^2 - \dfrac{\lambda}{2} \sum_{i=1}^D \sum_{j=1}^D w_{ij} (\eta_j-\eta_i)^2 S''(z_j+t \eta_j-z_i-t \eta_i),
$$
which gives
$$
\dfrac{d^2}{dt^2}E_{\lambda}(\ve{z}+t\boldsymbol{\eta};\ve{y})_{|t=0} =  \| \boldsymbol{\eta} \|^2 -\dfrac{\lambda}{2}  \sum_{i=1}^D \sum_{j=1}^D w_{ij} (\eta_j-\eta_i)^2 S''(z_j-z_i),
$$
As $S$ is twice differentiable and convex, we have, by definition, that $S'' \geq 0$. We deduce that $\dfrac{d^2}{dt^2}E_{\lambda}(\ve{z}+t\boldsymbol{\eta};\ve{y})_{|t=0}>0$ if $\lambda < 0$, which shows that $E_{\lambda}$ is strictly convex.\\
\\
Let us now assume that $\lambda>0$. We have that
\begin{align*}
\dfrac{d^2}{dt^2}E_{\lambda}(\ve{z}+t\boldsymbol{\eta};\ve{y})_{|t=0} & \geq \| \boldsymbol{\eta} \|^2 - \dfrac{\lambda}{2} K \sum_{i=1}^D \sum_{j=1}^D w_{ij} (\eta_j - \eta_i)^2. 
\end{align*}
Then, from the triangle inequality
$$
(\eta_j-\eta_i)^2 \leq \eta_j^2+\eta_i^2 +2 |\eta_j| |\eta_i|,
$$
we have
\begin{equation}
\label{eq: convexity 1}
\dfrac{d^2}{dt^2}E_{\lambda}(\ve{z}+t\boldsymbol{\eta};\ve{y})_{|t=0}  \geq \| \boldsymbol{\eta} \|^2 - \dfrac{\lambda}{2} K \sum_{i=1}^D \sum_{j=1}^D w_{ij} (\eta_j^2+\eta_i^2) - \dfrac{\lambda}{2} K \sum_{i=1}^D \sum_{j=1}^D w_{ij} |\eta_j| |\eta_i|
\end{equation}
The Cauchy-Schwartz inequality applied to the vectors $\ve{a},\ve{b} \in \mathbb{R}^{D \times D}$:
$$
a_{ij} = \sqrt{w_{ij}} |\eta_i| \qquad b_{ij} = \sqrt{w_{ij}} |\eta_j|
$$
gives $\langle \ve{a},\ve{b} \rangle \leq \|\ve{a}\|\|\ve{b}\|$, i.e.
\begin{align}
\label{eq: convexity 2}
\sum_{i=1}^D \sum_{j=1}^D w_{ij} |\eta_i| |\eta_j|  & \leq \left( \sum_{i=1}^D \sum_{j=1}^D w_{ij} \eta^2_j \right)^{1/2} \left( \sum_{i=1}^D \sum_{j=1}^D w_{ij} \eta^2_i \right)^{1/2} \nonumber \\
& \leq \| \boldsymbol{\eta} \|^2
\end{align}
We deduce from (\ref{eq: convexity 1}) and (\ref{eq: convexity 2}) that
\begin{align*}
\dfrac{d^2}{dt^2}E_{\lambda}(\ve{z}+t\boldsymbol{\eta};y)_{|t=0}  & \geq \| \boldsymbol{\eta} \|^2 - \lambda K \| \boldsymbol{\eta} \|^2 - \lambda K \| \boldsymbol{\eta} \|^2 \\
& \geq \| \boldsymbol{\eta} \|^2 (1- 2 \lambda K) \\
& \geq \| \boldsymbol{\eta} \|^2>0 \quad \textrm{if} \quad\lambda <  \dfrac{1}{2K}. 
\end{align*}
\end{proof}
\begin{corollary}
    \label{methods:corollary:existence_and_uniqueness}
Under the assumption that $S$ is positive, even and satisfies that $S(z) \leq A z^2$, for some  $A \geq 0$, then Eq.~\eqref{eq:z} possesses a solution for $\lambda <  \dfrac{1}{4A}$. Moreover, if $S$ is also twice differentiable and convex with second order derivative bounded above by $K>0$, then the solution is unique for $\lambda <  \dfrac{1}{\max (4A,2K)}$.   
\end{corollary}

\begin{proof}
According to Prop.~\ref{proposition:existence}, the energy $E_{\lambda}(\cdot ;M\ve{x}-\ve{b})$ is coercive for $\lambda <  \dfrac{1}{4A}$ if $S$ is positive, even and satisfies that $S(z) \leq A z^2$. Together with the continuity of $E_{\lambda}(\cdot ;M\ve{x}-\ve{b})$, it guarantees the existence of a minimum of $E_{\lambda}(\cdot ;M\ve{x}-\ve{b})$. Then, by differentiability of $E_{\lambda}(\cdot ;M\ve{x}-\ve{b})$, its minima correspond to the points vanishing its gradient, and these latter correspond to solutions of Eq.~\eqref{eq:z} according to Prop.~\ref{prop: relation solutions equation-critical points energy}. Moreover, if $S$ is also twice differentiable and convex with second order derivative bounded above by $K>0$, then the energy is also strictly convex if $\lambda <  \dfrac{1}{2K}$ according to Prop.~\ref{prop: strictly convexity}. We deduce that $E_{\lambda}(\cdot ;M\ve{x}-\ve{b})$ is coercive and strictly convex for $\lambda < \min \left( \dfrac{1}{4A}, \dfrac{1}{2K} \right)= \dfrac{1}{\max(4A,2K)}$, which guarantees the uniqueness of the minimum of the energy, which shows the uniqueness of the solution of Eq.~\eqref{eq:z}. 
\end{proof}

\subsubsection*{Practical computation of the solution}
 \label{sec:methods:practical_computation_of_the_solution}

\begin{corollary}
    \label{methods:corollary:IBNN_calculation_of_solution}
    Under the assumptions of Prop.~\ref{proposition:existence} and for a well-chosen sequence $\{ \tau^{(n)}\}$, the sequence $\ve{z}^{(n)}=(z_{i}^{(n)})_{i=1}^D$ defined by
    \begin{equation}
        \label{eq:iterative_calculation_of_IBNN}
        z^{(n+1)}_i =
            z^{(n)}_i + \tau^{(n)}
            \Big(
                -\! z^{(n)}_i + \sum_{j=1}^{N} m_{ij} x_j \!-\! b_i
                \mLbin \lambda \sum_{k=1}^{D} w_{ik} \sigma(z^{(n)}_{k}\!-\!z^{(n)}_{i})
            \Big), \qquad i \in \{1,\ldots,D\} \, ,
        \, 
    \end{equation}
    converges;
    and, for $\phi$ continuous, the sequence $\Phi(\ve{z}^{(n)}) = (\phi(z_{i}^{(n)}))_{i=1}^D \in \R[D]$ converges to an output of the corresponding \ac{IBNN} layer. Moreover, under the assumptions of Corollary~\ref{methods:corollary:existence_and_uniqueness}, $\Phi(\ve{z}^{(n)})$
    converges to the unique output of the \ac{IBNN} layer layer from any initial guess $\ve{z}^{(0)}=(z_{i}^{(0)})_{i=1}^D$.
\end{corollary}

\begin{proof}
Under the assumptions of Prop.~\ref{proposition:existence}, the energy $E_{\lambda}(\cdot; M\ve{x}-\ve{b})$ possesses a minimum as it is continuous and coercive. Moreover,
the sequence (\ref{eq:iterative_calculation_of_IBNN}) corresponds to the gradient descent algorithm associated to this energy, where $\tau^{(n)}$ is the step size at iteration $n$. Then, the convergence of $\ve{z}^{(n)}$ to a critical point of $E_{\lambda}(\cdot; M\ve{x}-\ve{b})$ is guaranteed if $\tau^{(n)}$ is well-chosen (\eg $\tau^{(n)}$ satisfies the Wolfe conditions). As the critical points of $E_{\lambda}(\cdot, M\ve{x}-\ve{b})$ correspond to the solutions of Eq.~\eqref{eq:z}, it shows that the sequence $\Phi(\ve{z}^{(n)})$ converges to an output of an IBNN layer if $\Phi$ is continous. Moreover, under the assumptions of Corollary \ref{methods:corollary:existence_and_uniqueness}, there is uniqueness of the critical point of $E_{\lambda}(\cdot, M\ve{x}-\ve{b})$, which guarantees the convergence of the sequence $\Phi(\ve{z}^{(n)})$ to the unique output of an IBNN layer.
\end{proof}

This iterative algorithm represents an intuitive and natural approach to the output of the \ac{IBNN} layer. However, in the conditions where the solution of the system of coupled implicit equations of Eq.~\eqref{eq:z} exists and unique, the proposed layer is agnostic to the specific algorithm leading to such solution: the practical implementation of the layer could, therefore, be based on any fixed-point or root-finding algorithm able to efficiently solve the task, as discussed later in the context of the effective implementation used for the experimental evaluation.\\
\\

 
Finally, we give a family of functions satisfying the desirable properties aforementioned.
\begin{proposition}
    The function
    
    \begin{equation}
        \label{eq: functions Sn}
        S_p \, \colon \; z  \mapsto 
            \dfrac{\ln (\cosh (p z))}{p} \, , \; p \in \R[+]
    \end{equation}
    
    \noindent satisfies $S_p(z) \leq \frac{p}{2} z^2 
    $, is positive, even, twice differentiable, and convex, and its second order derivative is bounded above by $p$.
\end{proposition}

\begin{proof}
It is clear that $S_p$ is positive and even.
Moreover, the trigonometric inequality $\cosh(z) \leq e^{\frac{z^{2}}{2}}$~\cite{Bagul2018}
and the monotonicity of the logarithm yield




\begin{equation*}
    \dfrac{\ln (\cosh (p z))}{p} \leq  \frac{p}{2} z^2 \, .     
\end{equation*}

Finally, the function $S_p$ is clearly twice differentiable, and it is convex with second order derivative bounded above by $p$ as

\begin{equation}
    \label{eq:methods:sigma_in_family_tanh_px}
    S_p'(z) = \tanh (pz) \quad( \triangleq \sigma(z))
\end{equation}

\noindent and
$$
S_p''(z) = \dfrac{p}{\cosh^2(pz)}, 
$$
from which we obtain that
$$
0 < S_p''(z) \leq p.
$$
\end{proof}

\begin{corollary}
    \label{methods:corollary:tahn_pz}
    The family of functions $\sigma_p(z) = \tanh(pz)$ fulfills the requirements of Corollaries~\ref{methods:corollary:existence_and_uniqueness} and~\ref{methods:corollary:IBNN_calculation_of_solution} and
    therefore for the uniqueness of the solution of Eq.~\eqref{eq:z}. 
\end{corollary}

%% file: _methods__ibnn_matrix_notation.tex
\newcommand{\SupplMatJacobian}{\textcolor{red}{Suppl. Material - Analysis of the robustness of IBNN-based neural networks}}

\ac{IBNN} layers based on neurons according to Eqs.~\eqref{eq:v}-\eqref{eq:z} can be written in a compact manner using multivariate functions and using matrix notation: we briefly define it next because it greatly clarifies certain proofs and results included in the following paragraphs.

\begin{itemize}
    \item
        Let the affine function $\ma{A}: \R[N]\to\R[D]$ be
        \begin{equation}
            \ma{A}(\ve{x}) = \ma{M}\ve{x} - \ve{b} \, ,
            \label{eq:methods:def_A}
        \end{equation}
        \noindent with $\ma{A} \!\in\! \mathcal{M}_{D \times N}(\R)$ and $\ve{b} \!\in\! \R[D]$.
    \item
        Let the non-linear function $\ma{B}: \R[D]\to\R[D]$ be defined component-wise as
        $\ma{B}(\ve{z}) \!=\! \big( \ma{B}_i(\ve{z}) \big)_{i=1}^{D}$,
        where
        \begin{equation}
            \ma{B}_i(\ve{z}) \!=\! \sum_{k=1}^D w_{ik} \, \sigma (z_k-z_i) \, .
            \label{eq:methods:def_Bi}
        \end{equation}
    \item
        Let the function $\ma{F}: \R[D]\to\R[D]$ be implicitly defined, using $\ma{B}$ and with $\lambda$ as a hyperparameter, as
        \begin{equation}
            \ma{F}(\ve{y}; \lambda) = \big\{ \ve{z} \; | \; \ve{z} = \ve{y} \mLbin \lambda \ma{B}(\ve{z}) \big\} \, .
            \label{eq:methods:def_F}
        \end{equation}
    \item
        And let $\Phi : \R[D] \to \R[D]$ be defined using a 1D component-wise outer activation $\phi$ as
        \begin{equation}
            \Phi(\ve{z}) = (\,\phi(z_i)\,)_{i=1}^D \, .
            \label{eq:methods:def_Phi}
        \end{equation}
\end{itemize}

Through these, Eqs.~\eqref{eq:v}-\eqref{eq:z} defining \ac{IBNN} can be condensed into

\begin{equation}
    \IBNN = \Phi \circ \ma{F}(\cdot; \lambda) \circ \ma{A}
    \label{eq:methods:def_IBNN_matrix_form}
\end{equation}

\noindent while \ac{SM} can be condensed into

\begin{equation}
    \SM = \Phi \circ \ma{A} \, .
    \label{eq:methods:def_SM_matrix_form}
\end{equation}

%% file: _methods__expressivity.tex


\subsubsection*{Number of \acs{SM} neurons required to reproduce the operation of an \acs{IBNN} layer of $D$ neurons}
\label{sec:methods:number_of_sm_neurons_per_ibnn_neuron}


Due to the entanglement between their outputs and their implicit nature, the operations performed by \ac{IBNN} layers cannot be efficiently represented using \ac{SM} neurons, which act individually and whose operation is purely forward. To reinforce this idea, we will show that, even for infinitesimal $\lambda$, approximating the behavior of an \ac{IBNN} layer of $D$-neurons could only be achieved, under the \ac{SM} paradigm, with a number of neurons that is orders of magnitude higher.

To this end we will leverage the first-order Taylor expansion of the operation performed by the \ac{IBNN} layer around $\lambda=0$. First, and using implicit differentiation on the matrix notation of Eq.~\eqref{eq:methods:def_IBNN_matrix_form}, the derivative with respect to $\lambda$ is

\begin{equation}
    \diffp{\IBNN}{\lambda} = \der[\Phi] \, \diffp{\ma{F}}{\lambda}
    , \qquad \textrm{with} \qquad
    \diffp{\ma{F}}{\lambda} = - \Big( \id{D} \pLbin \, \lambda \, \der[\ma{B}] \Big)^{\!-1} \ma{B}
    \, ,
\end{equation}

\noindent which yields the expansion

\begin{equation}
    \widetilde{\IBNN}(\ve{x}) \;=\;
    \IBNN_{|\lambda=0}(\ve{x}) + \lambda \left. \diffp{\IBNN}{\lambda} \right|_{\lambda\!=\!0}
    =\;
    \Phi \Big( \ma{A}(\ve{x}) \Big) \mLbin \, \lambda \; \der[\Phi] \; \ma{B} \Big( \ma{A}(\ve{x}) \Big) \, .
    \label{eq:methods:ibnn_taylor_with_respect_to_lambda}
\end{equation}

\noindent The calculation of the linearization~\eqref{eq:methods:ibnn_taylor_with_respect_to_lambda} of the \ac{IBNN} layer, composed of 2 summands, would require the following:

\begin{itemize}
    \item
        The $i$-th component of the first summand would require the calculation of the vector multiplication (plus offset) $\ma{A}_i(\ve{x})=\veT{m}_i \ve{x} - b_i$ plus the application of the non-linearity $\phi$, which coincides with operation of a \ac{SM} neuron.
    \item
        The $i$-th component of the second summand would require the product of $\ma{B}_{i}(\ma{A}(\ve{x}))$ with a scalar factor, since the matrix $\der[\Phi]$ is diagonal. The calculation of $\ma{B}_{i}(\ma{A}(\ve{x}))$ (see Eq.~\eqref{eq:methods:def_Bi}) entails calculations involving all the $D$ components $\ma{A}_k(\ve{x})$, $k\in\{1,\ldots,D\}$. In particular, and in line with the discussion in~\cite{Bertalmio2020Evidence} for convolutional operations, (even if $\ma{A}(\ve{x})$ is presumed precalculated due to its implication in the first summand) the calculation of $\ma{B}_{i}$ requires, first, the calculation of all `intermediate' signals $\sigma(\ma{A}_k(\ve{x}) - \ma{A}_i(\ve{x}))$, $\forall k\in\{1,\ldots,D\}$, requiring each a \ac{SM} neuron with output nonlinearity $\sigma$, and then the application of another \ac{SM} neuron (without output nonlinearity) to linearly combine them, resulting into $D+1$ neurons according to the \ac{SM}.
\end{itemize}

\noindent Therefore, even for this first-order approximation of a layer of $D$ \ac{IBNN} neurons, which would neglect further nonlinear behavior that that the real \ac{IBNN} layer would have, would require at least $D(D+2)$ neurons according to the \ac{SM} paradigm.


\subsubsection*{A \acf{UAT} \acused{UAT} for \ac{IBNN}}
\label{sec:methods:uat_IBNN}


A whole panoply of \acp{UAT} regarding the ability of the \ac{SM}-layers to approximate arbitrary functions exists~\cite{Cybenko1989, Pinkus1999}, differing on their respective assumptions about the type of network the \ac{SM} layer is part of. The most widespread version of the \ac{UAT} for the \ac{SM} refers to networks composed of a single hidden {SM} layer with a linear combination of its outputs, that is, the class of functions

\begin{equation}
    \mathcal{H}_{\SM}
    =
    \Big\{
        \ve{x} \mapsto 
        \ve{c} \cdot \SM(\ve{x}; \ma{M}, \ve{b}) = \ve{c} \cdot \Phi \big( \ma{M} \ve{x} -\ve{b} \big) 
    \Big\} \, ,
    \label{eq:class_functions_SM}
\end{equation}

\noindent where $\Phi_i = \phi$ is a non-polynomial activation and where $\ve{x} \in X \subset \R[N]$, $X$~compact (note its correspondence with the \ac{SM} neurons defined by Eqs.~\eqref{eq:u}-\eqref{eq:y}), and its proof consist in showing that the set~\eqref{eq:class_functions_SM} is dense in the space of continuous functions $\mathcal{C}(X)$~\cite{Pinkus1999}. (Note that we have explicitly included the matrix $\ma{M}$ and offset vector~$\ve{b}$ in the notation of the \ac{SM} in Eq.~\eqref{eq:class_functions_SM}.)

We will analogously prove that the class of functions 

\begin{equation}
    \mathcal{H}_{\IBNN^{\prime}}
    =
    \bigg\{
        \ve{x} \mapsto 
        \tilde{\ve{c}} \cdot \SM\Big( \IBNN(\ve{x}; \ma{M}, \ve{b}, \lambda); \tilde{\ma{M}}, \tilde{\ve{b}} \Big)
        =
        \tilde{\ve{c}} \cdot \Phi \Big( \tilde{\ma{M}} \, \IBNN(\ve{x}; \ma{M}, \ve{b}, \lambda) - \tilde{\ve{b}} \Big)
    \bigg\} \, ,
    \label{eq:class_functions_IBNN}
\end{equation}

\noindent again with a non-polynomial activation $\phi$ and defined for $\ve{x} \in K \subset \R[N]$ compact, is dense in the space of functions $\mathcal{C}(X)$.
Note that
$\mathcal{H}_{\IBNN^{\prime}}$
corresponds to an \ac{IBNN} layer followed by a \ac{SM} layer and a linear combination of the outputs of the latter.

\begin{proposition}
    For every $f \in \mathcal{C}(X)$ and every $\epsilon > 0$ there exists $\tilde{f} \in 
    \mathcal{H}_{\IBNN^{\prime}}$
    such that

    \begin{equation*}
        \sup_{\ve{x}\in X} \; \Big| f(\ve{x}) - \tilde{f}(\ve{x}) \Big| \;<\; \epsilon \, .
    \end{equation*}
    
\end{proposition}

\begin{proof}
    By the \ac{UAT} for $\mathcal{H}_{\SM}$, for any given $f\in\mathcal{C}(X)$ there exist $\ve{c}$, $\ma{M}$, and $\ve{b}$ such that the approximating function
    $g(\ve{x}) = \ve{c} \cdot \Phi(\ma{M}\ve{x}-\ve{b})$
    satisfies

    \begin{equation*}
        \sup_{\ve{x}\in X} \; \Big| f(\ve{x}) - g(\ve{x}) \Big| \;<\; \frac{\epsilon}{2} \, .
    \end{equation*}

    For this proof, let us assume that the nonlinearity $\phi$ is injective; although $\ReLU$, the prototypical activation function and the basis for the experiments in this manuscript, is in fact not injective, we could assume, without any loss of generality for the purpose of this proof, a family of injective softplus functions with uniform convergence to it. This allows the definition and use of the left inverse $\Phi_{\mathrm{left}}^{\!-1}$ such that $\Phi_{\mathrm{left}}^{\!-1} \circ \Phi = \mathrm{Id}$. In addition, for $\lambda$ where $\ma{F}(\cdot; \lambda)$ is single valued and is injective, the $\ma{L}(\cdot; \lambda): \R[D]\to\R[D]$ such that

    \begin{equation}
        \ma{L}(\ve{z}; \lambda) = \ve{z} + \lambda \ma{B}(\ve{z}) \, ,
        \label{eq:methods:def_L}
    \end{equation}
    
    \noindent is well defined and is, it and straightforwardly (see Eq.~\eqref{eq:methods:def_F}) the left-inverse of $\ma{F}(\cdot; \lambda)$ and thus $\ma{L}(\cdot; \lambda) \circ \ma{F}(\cdot; \lambda) = \mathrm{Id}$. Using these two maps, we can define the function $\ma{H}(\cdot; \lambda): \R[D]\to\R[D]$ as $\ma{H}(\cdot; \lambda) = \ma{L}(\cdot; \lambda) \circ \Phi_{\mathrm{left}}^{\!-1}$, which (represents by construction a partial, left inverse of the \ac{IBNN} layer and) yields 

    \begin{equation}
        \ma{H} \Big( \IBNN(\ve{x}; \ma{M}, \ve{b}, \lambda) ; \; \lambda \Big) = \ma{A}(\ve{x}) = \ma{M} \ve{x} -\ve{b} \, .
        \label{eq:methods:partial_inverse_IBNN}
    \end{equation}
    
    Let us define a new function $h(\ve{v}) = \ve{c} \cdot \Phi( \, \ma{H}(\ve{v}; \lambda) \, )$ based on the parameter $\ve{c}$ and the function $\Phi$ of the approximating function $g(\ve{x})$ above and on the function $\ma{H}(\cdot; \lambda)$ just defined. The domain of definition $V$ of $h$ is also compact, since $V=\IBNN(X)$ and \ac{IBNN} is continuous. Additionally, $h$ is continuous, since it is the composition of continuous maps. Leveraging again the \ac{UAT} for the \ac{SM}, this time on the compact domain $V$, there exist new $\tilde{\ve{c}}$, $\tilde{\ma{M}}$, and $\tilde{\ve{b}}$ such that

    \begin{equation*}
        \sup_{\ve{v}\in V} \; \Big| h(\ve{v}) - \tilde{\ve{c}} \cdot \Phi(\tilde{\ma{M}} \ve{v} - \tilde{\ve{b}}) \Big| \;<\; \frac{\epsilon}{2} \, .
    \end{equation*}
    
    \noindent Since, using Eq.~\eqref{eq:methods:partial_inverse_IBNN}, we have

    \begin{equation*}
        h \Big( \, \IBNN(\ve{x}; \ma{M}, \ve{b}, \lambda) \, \Big)
        =
        \ve{c} \cdot \Phi \left( \,
            \ma{H} \Big( \IBNN(\ve{x}; \ma{M}, \ve{b}, \lambda) ; \; \lambda \Big)
        \, \right)
        =
        \ve{c} \cdot \Phi(\ma{M}\ve{x}-\ve{b}) = g(\ve{x}) \, .
    \end{equation*}

    \noindent Therefore, the function $h \circ \IBNN(\cdot; \ma{M}, \ve{b}, \lambda)$ yields the claim through the triangle inequality.
\end{proof}

%% file: _methods__comparative_analysis_of_robustness.tex
\subsubsection*{Reduced sensitivity of \ac{IBNN} with respect to input perturbations}
\label{sec:methods:reduced_sensitivity_with_respect to the input}

The following analysis, based on the matrix notation of Eqs.~\eqref{eq:methods:def_A}-\eqref{eq:methods:def_SM_matrix_form},
compares the behavior of \ac{IBNN} and \ac{SM} layers regarding their sensitivity to input perturbation on the assumption that \ac{IBNN} and \ac{SM} nets having identical affine transforms (Eq.~\eqref{eq:methods:def_A}) and scoring head serve to the same classification problem
(\textrm{Supplementary Information}).

\paragraph{Local perturbations}

What follows demonstrates that the output of the layers and of the classifier \acp{ANN} based on the newly proposed \ac{IBNN} neurons are less sensitive to infinitesimal perturbations of their inputs than their \ac{SM} counterparts, by virtue of the analysis of their Jacobian matrix.

Using implicit derivation~\cite{BaiKolter2019} under the assumption that $\ma{F}(\cdot; \lambda)$ (see Eq.~\eqref{eq:methods:def_A}) is a
differentiable 
function $\ve{z}(\ve{y})=\ma{F}(\ve{y}; \lambda)$ of its input $\ve{x}$ we can write

\begin{equation} 
    \diff{}{\ve{y}} \, \ve{z}(\ve{y}) =
        \id{D}  \mLbin \lambda \diff{\ma{B}}{\ve{z}} \diff{}{\ve{y}} \, \ve{z}(\ve{y}) \, .
    \label{eq:methods:implicit_derivation_z}
\end{equation}

\noindent Reordering the terms and using the more compact symbol $\der[(\cdot)]((\cdot))$ {to denote the} 
absolute derivative calculated at the point indicated by its subscript yields

\begin{equation} 
    \diff{\ve{z}}{\ve{y}} = \der[\ma{F}](\ve{y}) =  \Big( \id{D} \pLbin \, \lambda \, \der[\ma{B}](\ma{F}(\ve{y})) \Big)^{\!-1}
    \triangleq \ma{K}_{\infty |\ma{F}(\ve{y})}
    \, .
    \label{eq:methods:der_F_Kinf}
\end{equation}

\noindent In the following we will refer to this term as $\der[\ma{F}](\ve{y})$ or $\ma{K}_{\infty |\ma{F}(\ve{y})}$.

Using Eq.~\eqref{eq:methods:der_F_Kinf}, the Jacobian matrix of \ac{IBNN} can be written through the application of the chain rule to Eq.~\eqref{eq:methods:def_IBNN_matrix_form} as

\begin{equation} 
    \der[\,\IBNN](\ve{x})
        =   \der[\Phi]((\ma{F} \circ \ma{A})(\ve{x})) \;
            \ma{K}_{\infty |(\ma{F} \circ \ma{A})(\ve{x})} \;
            \der[\ma{A}]\, ,
    \label{eq:methods:DIBNN}
\end{equation}

\noindent where we have used that $\der[\ma{A}](\ve{x}) = \der[\ma{A}] = \ma{M}$, constant for the whole input space.
Note that, based on the vector formulation of Eq.~\eqref{eq:methods:def_A}, the Jacobian matrix of \ac{SM} is

\begin{equation} 
   \der[\,\SM](\ve{x})
   = \der[\Phi](\ma{A}(\ve{x})) \; \der[\ma{A}] = \der[\Phi](\ma{A}(\ve{x})) \; \id{D}\; \der[\ma{A}] \, ,
   \label{eq:methods:DSM}
\end{equation}

\noindent which highlights the analogies between $\der[\,\SM](\ve{x})$ and $\der[\,\IBNN](\ve{x})$ but also the decisive interposed term $\ma{K}_{\infty |(\ma{F} \circ \ma{A})(\ve{x})}$ present in the expression Eq.~\eqref{eq:methods:DIBNN} of the latter. For convenience and for the sake of clarity, we will henceforth drop the explicit indication of the calculation point of each derivative.

Taking a deeper look into the shape and behavior of $\ma{K}_{\infty}$, we have that

\begin{equation}
    \der[\ma{B}] = \Big(\der[\ma{B}]_{ij}\Big)_{i,j \in \{1,\ldots,D\}} \, ,
    \quad \textrm{with } \quad
    \der[\ma{B}]_{ij} = \left\{
        \begin{matrix}
            - \displaystyle \sum_{\substack{\forall k \neq i}} \Big[
                w_{ik} \; \sigma^{\prime}(z_k - z_i)
            \Big]
            & \; j = i \\
            \hphantom{- \displaystyle \sum_{\substack{\forall k \neq i}} \Big[}
                w_{ij} \; \sigma^{\prime}(z_j - z_i)
            \hphantom{\Big]}
            & \; j \neq i
        \end{matrix}
    \right. \; ,
    \label{eq:B_jacobian_per_component}
\end{equation}

\noindent which means that, by virtue of the assumptions $w_{ji} = w_{ij}$ and $\sigma$ odd, $\der[\ma{B}]$ is symmetric, which in turn implies that it is diagonalizable and that it can be written as $\der[\ma{B}] = \ma{Q} \, \Xi \, \maT{Q}$ for a certain real orthogonal matrix $Q$ and a real diagonal matrix $\Xi = \diag\{\xi_1, \ldots, \xi_D\}$.
The fact that, additionally, each element in the main diagonal of $\der[\ma{B}]$ is exactly the negative sum of the rest of elements of its row has interesting additional implications:

\begin{itemize}
    \item
        Columns are linearly dependent, and thus at least one eigenvalue is $0$ (and its corresponding eigenvalue is in fact the uniform vector $\frac{1}{\sqrt{D}}(1, \ldots, 1)$).
    \item
        Since $w_{ik}$ and $\sigma^{\prime}$ are non-negative, $\der[\ma{B}]$ is a symmetric diagonally dominant matrix with real non-positive diagonal elements, which makes it negative semidefinite and implies that $\xi_i \leq 0$, $\forall i \in \{1, \ldots, D\}$.
        (In fact, by virtue of Gershgorin circle theorem, all the eigenvalues are within the interval $\xi_i \in \left[\, -2\max_i |\der[\ma{B}]_{ii}|, 0 \,\right]$.
\end{itemize}

\noindent These implications result (ordering eigenvalues in decreasing order of magnitude ($|\xi_1| \geq \ldots |\xi_i| \geq |\xi_D|$) without any loss of generality) in
$\Xi = \diag\{\xi_1 ,\, \ldots ,\, \xi_i ,\, \ldots ,\, \xi_{D-1} ,\, 0\}$, where $\xi_i \leq 0$, $\forall i \in \{1, \ldots, D-1\}$.
The injection of the above properties of $\der[\ma{B}]$ into the expression of $\ma{K}_{\infty}$, including $\xi_{D} = 0$, yields

\begin{equation}
    \begin{split}
        \ma{K}_{\infty}
            & 
            = \Big( \id{D} \, \pLbin \, \lambda \, \ma{Q} \, \Xi \, \maT{Q} \Big)^{\!-1}
            = \ma{Q} \; \Big( \id{D} \pLbin \lambda \, \Xi \Big)^{\!-1} \maT{Q}
            \\ &
            = \ma{Q} \;
                \diag\left\{
                    \frac{1}{1\pLbin\lambda\xi_1} \, , \; \ldots \, , \; \frac{1}{1\pLbin\lambda\xi_i} \, , \; \ldots \, , \; \frac{1}{1\pLbin\lambda\xi_{D\!-\!1}} \, , \; 1
                \right\} \;
            \maT{Q}
            \, .
    \end{split}
    \label{eq:methods:DIBNN_based_on_eigenvalues}
\end{equation}

\noindent This expression directly represents the diagonalization of $\ma{K}_{\infty}$ and links its eigenvalues $k_i$ one-to-one to the from the respective eigenvalues $\xi_i$ of $\der[\ma{B}]$ as $k_i=\frac{1}{1\pLbin\lambda\xi_i}$.

Since, as previously discussed, $\xi_i \leq 0$, $\forall i \in \{1, \ldots, D\!-\!1\}$, we can conclude that any $\lambda \leqL 0$ would ensure that all the eigenvalues $k_i$ of $\ma{K}_{\infty}$ are in the range $(0, 1]$ and thus $|k_i| \leq 1$.
For $\lambda > 0$, however, $|k_i|\leq 1$ is only fulfilled depending on the relative values of $\xi_i$ and $\lambda$, and in particular for $\lambda \geqL \mLuni \frac{2}{\xi_i}$. In both $\lambda > 0$ and $\lambda < 0$ cases $\ma{K}_{\infty}$ has at least one eigenvalue exactly $1$.

The norm $\Vert \ma{K}_{\infty} \Vert_2$, representing the worst-case direction regarding vector stretch, is $1$; nevertheless, the directions corresponding to eigenvectors of eigenvalues strictly $\xi_i < 0$ are compressed by the effect of $\ma{K}_{\infty}$.
This implies that, $\forall \delta\ve{y}\in\R[D]$, $\Vert \ma{K}_{\infty} \, \delta\ve{y} \Vert_2 \, = k \Vert \delta\ve{y} \Vert_2$ for some $k \in (0,1]$ dependent on the specific $\delta\ve{y}$, which implies, in turn,

\begin{equation}
    \begin{split}
        \Vert \der[\,\IBNN] \; \delta\ve{x} \Vert_2
            & = \;
            \Vert \der[\Phi] \; \ma{K}_{\infty} \; \der[\ma{A}] \; \delta\ve{x} \Vert_2
            \\
            & \;\leq\;
            \Vert \der[\Phi] \Vert_2 \; \Vert \ma{K}_{\infty} \; \der[\ma{A}] \; \delta\ve{x}\Vert_2
            \;\leq\;
            \Vert \der[\Phi] \Vert_2 \; \Vert \der[\ma{A}] \; \delta\ve{x}\Vert_2
            \, ,
    \end{split}
    \label{eq:matrix_norm_DIBNN_delta_x}
\end{equation}

\noindent where the second line results from the definition of the norm $\Vert \cdot \Vert_2$, ensuring
$\Vert \Psi \, \delta\ve{x} \Vert_2 \leq \Vert \Psi \Vert_2 \Vert \delta\ve{x} \Vert_2$. Finally, from this result and the relationship between Eq.~\eqref{eq:methods:DIBNN} and~\eqref{eq:methods:DSM},
and under the assumption that both \ac{IBNN} and \ac{SM} share a similar $\ma{A}$ and are working under the same locally-linear region of the non-linear activation $\Phi$ and thus $\der[\Phi](\ma{A}(\ve{x})) = \der[\Phi]((\ma{F} \circ \ma{A})(\ve{x}))$,
we have

\begin{equation}
    \frac{ \Vert \der[\,\IBNN] \; \delta\ve{x} \Vert_2 }{ \Vert \delta\ve{x} \Vert_2 }
    \leq
    \frac{ \Vert \der[\,\SM] \; \delta\ve{x} \Vert_2 }{ \Vert \delta\ve{x} \Vert_2 } \, ,
    \label{eq:methods:inequality_DIBNN_vs_DSM}
\end{equation}

\noindent which would suggest, according to the reasoning in~\cite{Jakubovitz2018}, a naturally improved robustness against adversarial attacks.



The above analysis for individual \ac{IBNN} and \ac{SM} layers also applies to networks having a backbone composed of a sequence of them. Suppose that we have respective \ac{IBNN}- and \ac{SM}-based backbones of $L$ layers according to

\begin{equation}
    \IBNN^{(1:L)} = \IBNN^{(L)} \circ \cdots \circ \IBNN^{(l)} \circ \cdots \circ \IBNN^{(1)}
    \; , \quad \textrm{with} \;
    \IBNN^{(l)} = \Phi \circ \ma{F}^{(l)}(\cdot; \lambda^{(l)}) \circ \ma{A}^{(l)}
    \label{eq:method:IBNN_multilayer}
\end{equation}

\noindent and

\begin{equation}
    \SM^{(1:L)} = \SM^{(L)} \circ \cdots \circ \SM^{(l)} \circ \cdots \circ \SM^{(1)}
    \; , \quad \textrm{with} \;
    \SM^{(l)} = \Phi \circ \ma{A}^{(l)} \, .
    \label{eq:method:SM_multilayer}
\end{equation}

\noindent The application of the chain rule, along with Eq.~\eqref{eq:methods:DIBNN} and~\eqref{eq:methods:DSM}, yields, respectively,

\begin{equation}
    \begin{matrix*}[c]
        \der[\IBNN^{(1:L)}]
        & = &
            \der[\IBNN^{(L)}]
            & \cdots &
            \der[\IBNN^{(l)}]
            & \cdots &
            \der[\IBNN^{(1)}]
        \\
        & = &
            \der[\Phi^{(L)}] \; \ma{K}^{(L)}_{\infty} \; \der[\ma{A}^{(L)}]
            & \cdots &
            \der[\Phi^{(l)}] \; \ma{K}^{(l)}_{\infty} \; \der[\ma{A}^{(l)}]
            & \cdots &
            \der[\Phi^{(1)}] \; \ma{K}^{(1)}_{\infty} \; \der[\ma{A}^{(1)}] 
    \end{matrix*}
    \label{eq:method:der_IBNN_multilayer}
\end{equation}

\noindent and

\begin{equation}
    \begin{matrix*}[c]
        \der[\SM^{(1:L)}]
        & = &
            \der[\SM^{(L)}]
            & \cdots &
            \der[\SM^{(l)}]
            & \cdots &
            \der[\SM^{(1)}]
        \\
        & = &
            \der[\Phi^{(L)}] \; \id{D} \; \der[\ma{A}^{(L)}]
            & \cdots &
            \der[\Phi^{(l)}] \; \id{D} \; \der[\ma{A}^{(l)}]
            & \cdots &
            \der[\Phi^{(1)}] \; \id{D} \; \der[\ma{A}^{(1)}] \, .
    \end{matrix*}
    \label{eq:method:der_SM_multilayer}
\end{equation}

\noindent For each \ac{IBNN} layer with $\lambda^{(l)} \leq 0$ we have that, $\forall \delta\ve{y}\in\R[D]$,
$\Vert \ma{K}^{(l)}_{\infty} \, \delta\ve{y} \Vert_2 \, = k^{(l)} \Vert \delta\ve{y} \Vert_2$ for some $k^{(l)} \in (0,1]$ dependent on the specific $\delta\ve{y}$.
Using sequentially this fact along with the definition of the norm $\Vert \cdot \Vert_2$ yields

\begin{equation}
    \begin{split}
        \Vert \der[\IBNN^{(1:L)}] \; \delta\ve{x} \Vert_2
        &=
        \Big\Vert\;
            \der[\Phi^{(L)}] \; \ma{K}^{(L)}_{\infty} \; \der[\ma{A}^{(L)}]
            \;\cdots\;
            \der[\Phi^{(l)}] \; \ma{K}^{(l)}_{\infty} \; \der[\ma{A}^{(l)}]
            \;\cdots\;
            \der[\Phi^{(1)}] \; \ma{K}^{(1)}_{\infty} \; \der[\ma{A}^{(1)}]
            \;
            \delta\ve{x}
        \;\Big\Vert_2
        \\
        &\leq
        \Vert \der[\Phi^{(L)}] \Vert_2 \;
        \Big\Vert\;
            \ma{K}^{(L)}_{\infty} \; \der[\ma{A}^{(L)}]
            \;\cdots\;
            \der[\Phi^{(l)}] \; \ma{K}^{(l)}_{\infty} \; \der[\ma{A}^{(l)}]
            \;\cdots\;
            \der[\Phi^{(1)}] \; \ma{K}^{(1)}_{\infty} \; \der[\ma{A}^{(1)}]
            \;
            \delta\ve{x}
        \;\Big\Vert_2
        \\
        &=
        \;k^{(L)}\;
        \Vert \der[\Phi^{(L)}] \Vert_2 \;
        \Big\Vert\;
            \der[\ma{A}^{(L)}]
            \;\cdots\;
            \der[\Phi^{(l)}] \; \ma{K}^{(l)}_{\infty} \; \der[\ma{A}^{(l)}]
            \;\cdots\;
            \der[\Phi^{(1)}] \; \ma{K}^{(1)}_{\infty} \; \der[\ma{A}^{(1)}]
            \;
            \delta\ve{x}
        \;\Big\Vert_2
        \\
        &\hphantom{=}
        \hphantom{\;k^{(L)}\;
        \Vert \der[\Phi^{(L)}] \Vert_2 \;}
        \vdots
        \\
        &\leq
        \;k^{(L)}\; \cdots \;k^{(1)}\;
            \left(
                \Vert \der[\Phi^{(L)}] \Vert_2 \;\; \Vert \der[\ma{A}^{(L)}] \Vert_2
                \cdots
                \Vert \der[\Phi^{(1)}] \Vert_2 \;\; \Vert \der[\ma{A}^{(1)}] \Vert_2
            \right)
            \Vert \delta\ve{x} \Vert_2 \; .
    \end{split}
    \label{eq:method:norm_der_IBNN_multilayer}
\end{equation}

\noindent The fact that we have, analogously,

\begin{equation}
    \begin{split}
        \Vert \der[\SM^{(1:L)}] \; \delta\ve{x} \Vert_2
        &=
        \Big\Vert
            \der[\Phi^{(L)}] \; \der[\ma{A}^{(L)}]
            \;\cdots\;
            \der[\Phi^{(1)}] \; \der[\ma{A}^{(1)}]
            \;
            \delta\ve{x}
        \Big\Vert_2
        \\
        &\leq
            \left(
                \Vert \der[\Phi^{(L)}] \Vert_2 \;\; \Vert \der[\ma{A}^{(L)}] \Vert_2
                \cdots
                \Vert \der[\Phi^{(1)}] \Vert_2 \;\; \Vert \der[\ma{A}^{(1)}] \Vert_2
            \right)
            \Vert \delta\ve{x} \Vert_2 \; ,
    \end{split}
    \label{eq:method:norm_der_IBNN_multilayer}
\end{equation}

\noindent along with the fact that $k^{(l)} \leq 1$, $\forall l \in \{1, \ldots, L \}$,
suggests that the sensitivity of the \ac{IBNN}-based multilayer backbone to perturbations of its input is lower than its \ac{SM}-based counterpart.

\paragraph{Global perturbations}

What follows demonstrates that the output of the layers based on the newly proposed \ac{IBNN} neurons are, globally, less affected by perturbations of their inputs than their \ac{SM} counterparts. The demonstration stems from the functionals introduced to prove the existence and uniqueness of the solution of \ac{IBNN} (Eq.~\eqref{eq:methods:energy_function}, \textrm{Methods}).

\begin{proposition}
Let $\ve{x} \in \R[N]$, $\mathcal{N}_{\ve{x}}$ a neighborhood of $\ve{x}$, and $\ve{y} \in \mathcal{N}_{\ve{x}}$. Under the assumption that $\lambda \leq 0,$
the function $S$ in $f_{\lambda}$ is convex,
and $\Phi$  is a similarity in
$\mathcal{F}$
$\subset \mathbb{R}^D$
such that
$$
    \mathcal{F}
    \supset \{ \prox_{f_{\lambda}}(\ma{A}(\mathcal{N}_{\ve{x}}))  \} \cup  \{ \ma{A}(\mathcal{N}_{\ve{x}}) \},
$$
we have
$$
\| \IBNN(\ve{x}; \ma{M},\ve{b},\lambda) - \IBNN(\ve{y}; \ma{M},\ve{b},\lambda) \|
\;\leq\;
\| \SM(\ve{x} ; \ma{M},\ve{b}) - \SM(\ve{y} ; \ma{M},\ve{b}) \| \qquad \forall \ve{y} \in \mathcal{N}_{\ve{x}}. 
$$
\end{proposition}

\begin{proof}
If $\lambda \leq 0$ and
the function $S$ in Eq.~\eqref{eq:methods:energy_function}
is convex, then $-f_{\lambda}$ is convex. Together with the fact that $-f_{\lambda}$ is continuous and proper, we have that the variational problem
$\underset{z}{\arg \min} \, E_{\lambda}(z; \ma{A}(\ve{x}))$ is the proximal operator $\prox_{-f_{\lambda}}(\ma{A}(\ve{x}))$ of $-f_{\lambda}$ evaluated at $\ma{A}(\ve{x})$, which leads to the existence and uniqueness of the solution of Eq.~\eqref{eq:z} and consequently of $\IBNN(\ve{x};\ma{M},\ve{b},\lambda) \quad \forall \,\, \ve{x} \in \R[N]$.\\
\\
By property of proximal operators, we then have that
$$
\| \prox_{-f_{\lambda}}(\ma{A}(\ve{x})) - \prox_{-f_{\lambda}}(\ma{A}(\ve{y})) \|
\;\leq\;
\| \ma{A}(\ve{x}) - \ma{A}(\ve{y}) \| \qquad \forall \, \ve{x},\ve{y} \in \R[N]. 
$$
We deduce that, under the assumption that $\Phi$ is a similarity, \ie an isometry scaled by a constant, in $\mathcal{F} \subset \R[N]$ such that
$$
\mathcal{F} \supset \{ \prox_{-f_{\lambda}}(\ma{A}(\mathcal{N}_{\ve{x}}))  \} \cup  \{ \ma{A}(\mathcal{N}_{\ve{x}}) \} \, ,
$$ 
we have
$$
\| \Phi(\prox_{-f_{\lambda}}(\ma{A}(\ve{x})) - \Phi(\prox_{-f_{\lambda}}(\ma{A}(\ve{y})) \|
\;\leq\;
\| \Phi(\ma{A}(\ve{x})) - \Phi(\ma{A}(\ve{y})) \| \qquad \forall \ve{y} \in \mathcal{N}_{\ve{x}} \, ,
$$
\ie
$$
\| \IBNN(\ve{x};\ma{M},\ve{b},\lambda) - \IBNN(\ve{y};\ma{M},\ve{b},\lambda) \|
\;\leq\;
\| \SM(\ve{x} ; \ma{M},\ve{b}) - \SM(\ve{y} ; \ma{M},\ve{b}) \| \, .
$$
\end{proof}

\subsubsection*{Inconclusive evidence about accuracy-stability tradeoff, unlike for the \ac{SM}}
\label{sec:methods:bastounis}

Bastounis\etal\cite{Bastounis2021} demonstrated that multilayer \ac{SM}-based \acp{ANN} with continuous nonlinear activations such as ReLU achieving a successful accuracy become universally unstable (see \cite{Bastounis2021} Theorem 2.2 (ii)), which represents a formal proof that \ac{SM}-based nets are subject to an inherent tradeoff between accurate and stability.

Their proof depends on the standard deep neural network structure $\Psi$ whose $L$ layers follow the form of Eqs.~\eqref{eq:u}-\eqref{eq:y}; that is, and using the notation of Eq.~\eqref{eq:methods:def_A} and~\eqref{eq:methods:def_Phi}, it can be expressed as the composition
$\Psi=\ma{A}^{(L)} \circ \Phi \circ \ma{A}^{(L-1)} \circ \cdots\circ \Phi \circ \ma{A}^{(l)} \cdots\circ \Phi \circ \ma{A}^{(1)}$,
where $\Phi_i = \phi$ is a continuous nonlinear function applied coordinate-wise such as ReLU and $\ma{A}^{(l)}$ are affine functions. A key factor in their proof depends on this standard network structure affording a linearization on finite but arbitrarily large sets (\cite{Bastounis2021} lemma 5.5), a property which does not directly carry over to the implicit and highly nonlinear \ac{IBNN} formulation. This linearization is critical in the construction of arbitrarily large sets where the output of the standard model neural network with fixed dimensions fails to match the true classification function output (\cite{Bastounis2021} lemma 5.4). Thus the accuracy-stability tradeoff established in \cite{Bastounis2021} does not directly apply to the \ac{IBNN}.

\subsubsection*{\ac{IBNN} is resistant to single-neuron stealth attacks}
\label{sec:methods:stealth_attacks}

We define an \ac{ANN} composed of one single hidden layer according to the \ac{IBNN} paradigm and with uniform weights $w_{ik}$ so

\begin{align}
    z_i(\ve{x}) &= \veT{m}_{i} \ve{x} - b_i \,\mLbin\, \frac{\lambda}{D} \sum_{k=1}^D \sigma\!\left( z_k(\ve{x}) \!-\! z_i(\ve{x})\right) \\
    v_i(\ve{x}) &= \phi(z_i(\ve{x}))
\end{align}

\noindent and one single output layer so the output score of the resulting \ac{IBNN} \ac{ANN} is:

\begin{equation}
    s(\ve{x})=\sum_{i=1}^N c_i v_i(\ve{x}) \, .
\end{equation}

Following \cite{Tyukin2024}, we say that a successful single-neuron stealth attack on the \ac{IBNN} \ac{ANN} yielding the scoring function $s$ would be a modified \ac{IBNN} \ac{ANN} with the same artificial units of the original plus an extra \ac{IBNN} neuron and yielding a scoring function $s_a$ such, given a verification input set $\mathcal{V}$, known to the owner of the system $s$ but unknown to the attacker, and $\epsilon \geq 0$, $\Delta > 0$,

\begin{enumerate}
    \item
    for some input $\ve{x}_T$, known to the attacker but not to the owner of the network, $| s_a(\ve{x}_T) - s(\ve{x}_T)| \geq \Delta$, and
    \item
    for every input $\ve{x}_v \in \mathcal{V}$,  $| s_a(\ve{x}_v) - s(\ve{x}_v)| \leq \epsilon$.
\end{enumerate}

\noindent Condition~(1) implies that the attacker makes the `hacked' network $s_a$ respond to a trigger input $\ve{x}_T$ in a way that is different from what the original $s$ does, whereas condition~(2) implies that the system owner will not be able to tell that an attack has taken place, that the original network $s$ has been modified, because the responses of the hacked network $s_a$ on $\mathcal{V}$ are very similar to those of the original network $s$.


The hacked network $s_a$ is

\begin{align}
    s_a(\ve{x}) &=\sum_{i=1}^N c_i \hat{v}_i(\ve{x}) + c_a \hat{v}_{N+1}(\ve{x}) \\
    \hat{v}_i(\ve{x}) &= \phi( \hat{z}_i(I) ) \\
    \hat{z}_i(\ve{x}) &=
        \veT{m}_{i} \ve{x} - b_i \,\mLbin\, \frac{\lambda}{D} \sum_{k=1}^{D+1} \sigma\!\left( \hat{z}_k(\ve{x}) \!-\! \hat{z}_i(\ve{x}) \right)
\end{align}

\noindent Let $\ve{x}_T$ be a trigger, such that the magnitude of the attack it produces is $\Delta$; taking $\phi$ to be the ReLU function and assuming that the network units are working in the linear range of $\phi$, we have

\begin{equation}
    \Delta = | s_a(\ve{x}_T) - s(\ve{x}_T)| =
        \left| \sum_i^D c_i \Big( \hat{v}_i(\ve{x}_T) - v_i(\ve{x}_T) \Big) + c_a \hat{v}_{D+1}(\ve{x}_T) \right| \, .
\end{equation}

\noindent If $D \gg 1$ we may assume that adding a single neuron does not change significantly the sum $\sum_{k=1}^D  \sigma(z_j(\ve{x}) - z_i(\ve{x}))$, since only one term is added to the summation; therefore, we have $\hat{z_i} \simeq z_i$ and hence

\begin{equation}\label{eq:Delta}
    \Delta \simeq  \left| c_a \; \hat{z}_{D+1}(\ve{x}_T) \right| \, .
\end{equation}

\noindent Additional, we have

\begin{equation}\label{eq:ua}
    \hat{z}_{D+1}(\ve{x}_T) = \veT{m}_{D+1} \ve{x}_T \,\mLbin\, \frac{\lambda}{D} \sum_{k=1}^{D+1}  \sigma(\hat{z}_k(\ve{x}_T) - \hat{z}_{D+1}(\ve{x}_T)),
\end{equation}

\noindent where, for simplicity, we have chosen $b_{D+1}=0$. Combining Eqs.~\eqref{eq:Delta} and~\eqref{eq:ua} we get

\begin{align}
    \left| \frac{\Delta}{c_a} \right| &=
        \bigg| \, \veT{m}_{D+1} \ve{x}_T \,\mLbin\, \frac{\lambda}{D} \sum_{k=1}^{D+1} \sigma\Big(\hat{z}_k(\ve{x}_T) \pm  \frac{\Delta}{c_a} \Big) \, \bigg| \\
    \ve{m}_{D+1} &=
        \frac{\ve{x}_T}{|| \ve{x}_T ||^2 }
            ( \pm  \frac{\Delta}{c_a} \pm \frac{\lambda}{D} \sum_{k=1}^{D+1}  \sigma(\hat{z}_k(\ve{x}_T) \pm \frac{\Delta}{c_a}) ) \\
    \ve{m}_{D+1} &\propto \ve{x}_T
\end{align}

\noindent This proves that we can find a filter $\ve{m}_{D+1}$ so that the input $\ve{x}_T$ produces very different outputs in $s$ and $s_a$: the difference between the outputs of both networks has an amplitude of $\Delta$, which would fulfill the first requirement of the stealth attack.

Now we need to check the second requirement, \ie that the outputs for the verification set remain unaltered, or change very little for $\ve{x}_v \in \mathcal{V}$. Using analogous assumptions to the reasoning above, the difference $\Delta_v$ between $s_a(\ve{x}_v)$ and $s(\ve{x}_v)$ is

\begin{equation}
    \Delta_v \simeq \left| c_a \; \hat{z}_{D+1}(\ve{x}_v) \right| \, .
\end{equation}

\noindent We would like $\Delta_v$ to be negligible, \eg $\Delta_v \simeq 0$, which gives $ \hat{z}_{D+1}(\ve{x}_v) \simeq 0$, and from the definition of $\hat{z}_{D+1}$ we get

\begin{equation}\label{eq:condIv}
    \veT{m}_{D+1} \ve{x}_v = \frac{\lambda}{D} \sum_{k=1}^{D}  \sigma(z_k(\ve{x}_v))
\end{equation}

\noindent So, in order for the stealth attack to be actually stealth, undetectable, all the inputs in the verification set $\mathcal{V}$ should satisfy Eq.~\eqref{eq:condIv}. But this is an implicit equation on $\ve{x}_v$, whose solutions, if any, might be totally unrelated to the elements of $\mathcal{V}$, which is a set that is only known to the system owner, not to the attacker.

In conclusion, condition~(2) for the stealth attack is not satisfied, and therefore we can say that single-neuron stealth attacks on a \ac{IBNN} are unfeasible.

%% file: _methods__memorization.tex


Here we show that IBNN is less prone to memorization by following the procedure suggested by Garg\etal\cite{Garg2024},  where the tendency of an ANN to memorize its training data is evaluated by computing the curvature of the loss function with respect to the input: lower curvature means less memorization.

    Let $\Psi$ be a binary classifier with a single hidden layer and a \ac{FC}-based scorer $s: \R[D] \to \R$ according to
    
    \begin{equation*}
        s(\ve{x}) = \veT{c} \, \IBNN(\ve{x}) - c_0 \, ,
    \end{equation*}
    
    \noindent and a subsequent probability assignment according to $\Psi(\ve{x}) = P\big( s(\ve{x}) \big)$,
    where $\ve{c} \in \R[D]$, and $c_0 \in \R$ and where $P: \R \to [0,1]$ is a function providing the probability of the class $1$ (for two classes, with respective labels $0$ and $1$), \eg a softmax summarization making $y$ the probability of choosing the label $1$, complementary of the probability of choosing the label $0$. Additionally, let the loss for the training pair $(\ve{x}, y_{t})$ ($t$ for \emph{truth}) be the \ac{MSE}, that is,
    $L(\ve{x}; y_{t}) = (\Psi(\ve{x})-y_{t})^{2}$. The derivative of the loss with respect to the input is, by virtue of the chain rule,
    
    \begin{equation}
    \diffp{L_{\IBNN}}{\ve{x}} \; = \;
        2 \Big( \Psi(\ve{x})-y_{t} \Big) \;
        P^{\prime}(s(\ve{x})) \;
        \veT{c} \; \der[\,\IBNN](\ve{x}) \, .
        \label{eq:derivative_z_matrix_notation}
    \end{equation}
    
    \noindent An analogous expression $\diffp{L_{\SM}}{\ve{x}}$ would correspond to an analogous \ac{SM}-based network, \emph{mutatis mutandis}.
    
    
    The exact expression for $\der[\,\IBNN](\ve{x})$ in Eq.~\eqref{eq:derivative_z_matrix_notation} is provided by Eq.~\eqref{eq:methods:DIBNN} --- it is the product of the terms $\der[\Phi]((\ma{F} \circ \ma{A})(\ve{x}))$, $\ma{K}_{\infty |(\ma{F} \circ \ma{A})(\ve{x})}$, and $\der[\ma{A}]$, which can be approximated as follows:
    

    \begin{itemize}
    \item
        $\der[\Phi] \approx \id{D}$, assuming that the \ac{IBNN} outer activation is a $\ReLU$ and that the network is working within its linear range;
    \item
                $\ma{K}_{\infty |(\ma{F} \circ \ma{A})(\ve{x})} \approx \id{D} \mLbin \, \lambda \, \der[\ma{B}](\ma{A}(\ve{x}))$,
        assuming $|\lambda| \ll 1$, 
        $(\ma{F} \circ \ma{A})(\ve{x}) \approx \ma{A}(\ve{x})$
        and using the first-order Taylor expansion of Eq.~\eqref{eq:methods:der_F_Kinf}.
    \end{itemize}

    \noindent Accordingly, $\der[\,\IBNN](\ve{x})$ can be approximated as
    
    \begin{equation}
    \der[\,\IBNN](\ve{x}) \approx \Big( \id{D} \mLbin \, \lambda \, \der[\ma{B}](\ma{A}(\ve{x})) \Big) \der[\ma{A}]
    \end{equation}
    
    \noindent which, in scalar notation and identifying $z_i = \IBNN_i(\ve{x})$, is
    
    \begin{equation}
    \diffp{z_i}{x_j}
    \approx
    m_{ij} - \lambda \sum_k w_{ik} \sigma^{\prime} \Big( (\veT{m}_k - \veT{m}_i) \ve{x} \Big) (m_{kj} - m_{ij})  \, .
    \end{equation}
    
    Assuming that $\sigma$ is piece-wise linear with central slope $p$
    (\ie $\sigma(z)=-1$, $\forall z < -\frac{1}{p}$;
    $\sigma(z)=1$, $\forall z>\frac{1}{p}$;
    and $\sigma(z)=pz$, $\forall |z|<\frac{1}{p}$),
    that the weights $w_{ik}$ are uniform with $w_{ik} = \frac{1}{D}$, $\forall i, k \in \{1,\ldots,D\}$, and that the norm of $\veT{m}_i \ve{x}$ does not vary significantly for different units $i$ and thus $| (\veT{m}_k - \veT{m}_i) \ve{x} | \ll 1$, we can write
    
    \begin{equation}
        \begin{split}
            \diffp{z_i}{x_j}
            \, \approx \,
            m_{ij} - \lambda p \left( \sum_{k=1}^{D} \frac{m_{kj}}{D}\right) + \lambda \, p \, m_{ij}
            =
            (1 + \lambda p) m_{ij} - \lambda p \left( \frac{1}{D} \sum_{k=1}^{D} m_{kj}\right) \, .
        \end{split}
        \label{eq:par_ui_par_xk}
    \end{equation}
    
    We now make the common assumption that the neurons have a reduced receptive field, \ie each neural unit receives inputs from a small amount of coordinates of $\ve{x}$.
    Then, the matrix whose rows are the filters $\veT{m}_i$ is sparse, and 
    with bounded filter elements (e.g. $|m_{ij}| \leq 1$, $\forall i,j$) we get $|\frac{1}{D} \sum_{j=1}^{D} m_{kj}| \ll 1$, which, plugged in Eq.~\eqref{eq:par_ui_par_xk}, gives
    
    \begin{equation}
        \diffp{z_i}{x_j}
            \, \approx \,
            (1 + \lambda p) m_{ij} \, ,
        \label{eq:par_ui_par_xk_with_zero_mean}
    \end{equation}
    
    \noindent where, just as a note, $\lambda = 0$ would correspond to the \ac{SM}.
    
    With this relationship between derivatives with respect to the input for \ac{IBNN} and for \ac{SM}, and going back to Eq.~\eqref{eq:derivative_z_matrix_notation}, we have
    
    \begin{equation}
        \diffp{L_{\IBNN}}{x_j} \approx (1 + \lambda p) \diffp{L_{\SM}}{x_j} \, .
        \label{eq:derivative_loss_comparison}
    \end{equation}

    \noindent Since, according to~\cite{Garg2024}, the magnitude of the curvature of the loss function w.r.t. the input can be estimated as
    \begin{equation}
    \mathrm{Curv}(\ve{x}) \propto \left\Vert
        \diffp{L(\ve{x}+h\ve{g})}{\ve{x}} - \diffp{L(\ve{x}))}{\ve{x}}
    \right\Vert^{2}  \, ,
    \end{equation}
    
    \noindent where $\ve{g} \in \R[N]$ are vectors sampled from a Rademacher distribution, the above approximation yields
    \begin{equation}\label{eq:memo}
    \mathrm{Curv}_{\IBNN}(\ve{x}) = (1 + \lambda p)^{2} \; \mathrm{Curv}_{\SM}(\ve{x})
    \end{equation}
    
    Given that that the magnitude of the curvature is a measure of memorization~\cite{Garg2024}, 
    Eq. \ref{eq:memo} shows that with a proper choice of $(\lambda, p)$ (s.t. $|1 + \lambda p| < 1$) \acs{IBNN} memorizes less than a \acs{SM} network.

%% file: _methods__implementation_of_IBNN.tex

Trainable \ac{IBNN} layers have been custom-implemented in Python using the design principles and the utilities provided by PyTorch~\cite{Paszke2019PyTorch} for both \ac{ANN} training and evaluation. The implemented \ac{IBNN} layers are analogous to and mostly interchangeable with the \ac{SM} layers available in PyTorch, thus easing comparisons.

Two types of \ac{IBNN} layers, convolutional and \acf{FC},  have been implemented, which are, respectively, analogous to convolutional and \ac{FC} layers according to the \ac{SM} paradigm. The two types differ in the spatial limitations imposed on the linear filters represented by the elements $\{\ve{m}_i, b_i\}$ (see, respectively, Eq.~\eqref{eq:y} and~\eqref{eq:z}): convolutional layers have been achieved by shaping the filter elements $\{\ve{m}_i, b_i\}$
to correspond to a square mask and tying them across all the neurons of the same layer.

As highlighted before, the output $\ve{v}$ of an \ac{IBNN} layer for a certain input $\ve{x}$ corresponds to the solution of a system of coupled implicit equations (Eqs.~\eqref{eq:v}-\eqref{eq:z}), which represents a challenge for both forward and backward (\ie training) operations. The library TorchDEQ~\cite{TorchDEQ2023}, providing both operations, is the engine of the implicit calculations of our implementation of the \ac{IBNN} layer:

\begin{itemize}
    \item
    The forward operation of the \ac{IBNN} layer, which is an implicit, but more specifically, fixed-point calculation, has been performed using the default forward solvers provided by TorchDEQ, in particular Broyden's method~\cite{TorchDEQ2023}.
    Additionally, a forward solver specifically tailored to \ac{IBNN} and based on the minimization of Eq.~\eqref{eq:methods:energy_function} through the procedure of Eq.~\eqref{eq:iterative_calculation_of_IBNN} has also been implemented. Neither showed convergence rates that would advise a preference for one over the other, and they have been used interchangeably.
    \item
    The backward operation, required for training, has been performed using the default forward solvers provided by TorchDEQ, in particular Broyden's method~\cite{TorchDEQ2023}.
\end{itemize}

Our implementation of the \ac{IBNN} layers allows for networks with independent $\lambda$s at each layer, and for fixed or trainable $\lambda$s: the latter option has been used in certain experiments as a sort of tentative hyperparameter optimization.

%% file: _methods__experimental_details.tex
\acused{AN}
\acused{CNN}
\acused{FC}
\acused{IBNN}
\acused{SM}

\subsubsection*{Datasets}
\label{sec:methods:datasets}

The image classification experiments were based on three datasets of increasing difficulty level, each one containing $10$ classes: Fashion-MNIST~\cite{Xiao2017FashionMNIST}, SVHN~\cite{Netzer2011SVHN}, and CIFAR-10~\cite{Krizhevsky2019CIFAR10}. These datasets have been used from their respective built-ins in the module Torchvision~\cite{Torchvision2016} of PyTorch~\cite{Paszke2019PyTorch}.
Fashion-MNIST consists of $28 \times 28$ grayscale images of clothing, for a total of $60,000$ training and $10,000$ test images; their comprised classes are T-shirt/top, trousers, pullover, dress, coat, sandal, shirt, sneaker, bag, and ankle boot.
SVHN consists of $32 \times 32$ color images of digits from $0$ to $9$, obtained as croppings of real street images, for a total $73,257$ training and $26,032$ test images.
And CIFAR-10 consists of $32 \times 32$ color images of animals and vehicles, for a total of $50,000$ training and $10,000$ test images; their comprised classes are airplane, automobile, bird, cat, deer, dog, frog, horse, ship, and truck.

The visualizations in Fig.~\ref{fig:visualization_2d_lambda_p} correspond to a custom 2D dataset, generated using a custom defined mask containing the two concentric arrows hosting the 2D samples of the two separated classes: the samples of each class were generated by `accept-reject sampling' from a 2D uniform distribution with a support comprising their respective masks. The resulting 2D samples where generated in the form of 2-pixel images for compatibility with the image-oriented implementation of the classifiers.

\subsubsection*{Network architectures used in the experimental evaluation}
\label{sec:methods:network_architectures}

The classifier \acp{ANN} used in the presented experimental evaluation, which we will refer to as  \acfp{UCN}\acused{UCN} and which will be denoted using the convention \acs{UCN}$(L, C)$, is a simplified architecture composed of:

\begin{itemize}
    \item
        A backbone formed of $L$ sequential blocks of analogous characteristics (hence the nomenclature `uniform'). Each block $l \in \{1,\ldots,L\}$ contains a convolutional layer, defined according to the \ac{SM} or \ac{IBNN} depending on the case, and a batch normalization layer for its output.
        In all cases, the convolutional layer maintains (through zero padding) the spatial size of its input image but produces $C$ channels at its output, and the side of the convolutional filters is in all cases the $15$\% of the size of its input image.
        $\ReLU$ is used as the output nonlinearity $\phi$ (Eq.~\eqref{eq:u} and~\eqref{eq:v}) of all layers.
    \item
        An output scoring head composed of one single \ac{FC} layer (according to the \ac{SM}) with as many neurons as classes contained in the dataset.
\end{itemize}

\noindent For convolutional layers according to \ac{IBNN}, additionally, uniform weights $w_{ik}=\frac{1}{D}$ and a dendritic nonlinearity $\sigma(z)=\tanh(p z)$ with $p\!=\!10$ is used. This value of $p$ was chosen for two reasons: it is an intermediate value regarding boundary flexibility (see Fig.~\ref{fig:visualization_2d_lambda_p}); and it ensures existence and uniqueness of \ac{IBNN} in the considered ranges of $\lambda$ (see Eq.~\eqref{eq:methods:sigma_in_family_tanh_px}).
The parameter $\lambda_l$ of each layer $l \in \{1,\ldots,L\}$ is set or trained individually.

All experiments shown in \textrm{Results} correspond to the described \acs{UCN} architecture.
However, experiments corresponding to different datasets and aimed aspects differ in $L$ and $C$. For each dataset, all experiments gravitate around a configuration \acs{UCN}$(L_d,C_d)$ whose $L_d$ and $C_d$ were set in order to achieve base accuracies in the range of $65-90 \%$. This central configuration was, for each dataset:
for Fashion-MNIST, \acs{UCN}$(1, 3)$;
for SVHN, \acs{UCN}$(2, 3)$;
and for CIFAR-10, \acs{UCN}$(3, 8)$.
\noindent All the experiments correspond to said central configurations but for two exceptions: experiments related to expressivity compare network sizes by varying the number of channels, and thus correspond for each dataset to the $L$ indicated above but capture different values of $C$ (see Fig.~\ref{fig:expressivity_all_datasets}); and the result included in Fig.~\ref{fig:learning_speed_all_datasets} for CIFAR-10, corresponding to \acs{UCN}$(3, 16)$ is the only departure from these configurations: although both $C\!=\!8$ and $C\!=\!16$ similarly reproduce the same positive trend regarding learning speed, it was determined that the latter illustrates the concept better.

The visualizations in Fig.~\ref{fig:visualization_2d_lambda_p}, however, correspond to a network with a backbone of a single hidden layer of \ac{FC} type composed of $3$ neurons (according to the \ac{SM} or \ac{IBNN} depending on the case) with a $\ReLU$ output nonlinearity and a batch normalization layer, followed by a single layer head/unit leading to a binary classification.

\subsubsection*{Training and evaluation details for \ac{IBNN} and \ac{SM}-based networks}
\label{sec:methods:_training_details_for_IBNN}

Training of both \ac{SM}- and \ac{IBNN}-based \acp{ANN} was performed using \ac{SGD}. The adoption of \ac{SGD}, instead of other more dynamic optimizers such as Adam, is motivated by two reasons: the latter requires the calculation of higher-order moments and therefore conveys a larger computational burden; and the simplicity of \ac{SGD} has been regarded as more controllable by experimenters with respect to the more obscure optimization steps of the latter. A scheduler composed of a linear warmup from $0.001$ to $0.01$ in $40$~epochs and a linear decay symmetric to the former was used for the \ac{LR}. This scheduler was chosen because it represents a favorable trade-off between stability with respect to the initial random weights of the network, reasonable speed due to the higher \ac{LR} during the central epochs, and better convergence due to the final cooldown.


A warmup training step of $2$ epochs was incorporated prior to the training of both \ac{SM}- and \ac{IBNN}-based \acp{ANN}. In the specific case of the \ac{IBNN}-based nets, this warmup was performed on a \ac{SM}-based surrogate with an identical architecture to that of the target \ac{IBNN}; 
the training of the \ac{IBNN} net itself starts by initializing its weights with the weights resulting from the warmup of the surrogate (see Eqs.~\eqref{eq:y} and~\eqref{eq:z}, and consider also the batch normalization and head of the \ac{UCN}). This surrogate-based warmup has been introduced to ensure favorable initializations of \ac{IBNN}-based \acp{ANN} and to avoid those rare cases where random initializations without such warmup did not allow training to progress to viable classifiers.

Each model (labeled in each case as \ac{IBNN} or \ac{SM}) in each experiment in \textrm{Results} was obtained as the median of several instances of the same model, trained independently from distinct random initial weights and dataset reordering: the results addressing expressivity (Fig.~\ref{fig:expressivity_all_datasets}) and memorization (Fig.~\ref{fig:memorization_all_datasets}) aggregate $10$ instances for the \ac{SM} nets and $5$ for the \ac{IBNN} nets; the results addressing decreasing percentages of the training data (Fig.~\ref{fig:partial_training_data}) and learning speed (Fig.~\ref{fig:learning_speed_all_datasets}) aggregate $20$ instances for the \ac{SM} nets and $5$ for the \ac{IBNN} nets. Corresponding confidence intervals of $90$\% are included in all cases. Additionally, the results illustrating learning speed (Fig.~\ref{fig:learning_speed_all_datasets}) perform a running average of $3$~epochs.

\subsubsection*{Adversarial attack generation and evaluation details}
\label{sec:methods:_adversarial_attack_generation_and_evaluation}

The respective \ac{AR} of \ac{SM}- and \ac{IBNN}-based nets has been assessed using two \acp{AA} of very different nature to better characterize the abilities of the proposal: \ac{PGD}~\cite{MadryPGD2018} and Pixle~\cite{PomponiPixle2022}. In both cases, attacks have been generated using the library TorchAttacks~\cite{Kim2021TorchAttacks} from the validation subset of each dataset.

Originally, \ac{PGD} is a white-box attack that presumes the full availability of the target \ac{ANN} and its gradient with respect to variations in its input. It performs a number of gradient-based iterations and limits its operation to a maximum distance with respect to the image from which it is calculated, denoted as $\epsilon$. However, our implementation based on \ac{IBNN} prevents the calculation of such a gradient and thus the feasibility of direct \ac{PGD} attacks, analogous to the gradient masking/obfuscation defenses against \acp{AA}. Nevertheless, the literature has shown that attacks generated using differential surrogate models are usually effective for attack generation~\cite{Biggio2018, Papernot2017}. Accordingly, while \ac{PGD} attacks for the \ac{SM}-based \acp{ANN} have been generated in a purely white-box manner, the \ac{PGD} attacks for the \ac{IBNN}-based \acp{ANN} have been produced using a \ac{SM}-based surrogate model, as follows: each trained \ac{IBNN} layer has been replaced by the \ac{SM} layer resulting from simply removing the nonlinear bias of Eq.~\eqref{eq:z} (identical to setting $\lambda=0$ for the layer). This type of surrogate model has been selected because it is deemed to be the most likely surrogate chosen by a potential attacker due to its similarity to the target \ac{IBNN} net.

Conversely, Pixle is purely black-box, working only on the classification probabilities reported by the \ac{ANN} for a given query image, and has been used to directly attack both \ac{SM} and \ac{IBNN} without any need for surrogates. In this way, it provides a complementary view to the results obtained using \ac{PGD} \emph{via} a surrogate. Pixle gradually shuffles an increasing number of pixels from the original image as its so-called restarts advance, and optimizes the effectiveness of the newly scrambled pixels at each restart for a certain number of iterations.

All Pixle experiments presented in this manuscript correspond to an identical attack with a maximum number of $10$ restarts with $5$ iterations each, with the generation of pixel candidates based on similarity and from $3 \times 3$ image patches~\cite{PomponiPixle2022}. The \ac{PGD}-based attacks presented in this manuscript perform $10$ gradient iterations and sweep across ranges of attack intensities, from $\epsilon=0$ to a certain maximum that depends on the dataset under study, resulting in the corresponding security evaluation curves~\cite{Biggio2018} for each experimental condition.

The models subject to adversarial attacks were trained according to the same principles, architectures, and hyperparameters indicated above, and, as a result, several instances of each model were obtained and evaluated on the clean validation dataset and under attack. The results addressing both \ac{PGD} (Fig.~\ref{fig:robustness_pgd}) and Pixle (Fig.~\ref{fig:robustness_pixle}) aggregate, using the median, $5$ model instances for the \ac{SM} nets and $5$ for the \ac{IBNN} nets.
Since confidence interval estimation would be unreliable for the limited amount of available samples, adversarial attack experiments show the complete range of obtained results (equivalent to confidence intervals of $100$\%).

\subsubsection*{Memorization: experimental evaluation}
\label{sec:methods:_memorization_experimental_evaluation}

Memorization occurs when models start learning individual training samples instead of the underlying statistical distribution of data: therefore, its occurrence must become apparent on the evolution of the network on its validation dataset.

We devised, consequently, the following experiment: we deliberately contaminated the training subset of a certain dataset with different proportions of wrong samples, generated by simply scrambling their labels, while keeping the validation subset unscathed; for each considered proportion we trained networks using the contaminated data and took the epoch when the corresponding validation accuracy started to permanently decline as a sign of memorization and overfitting; and we regard such maximum validation accuracy as the level of generalization achievable for an \ac{ANN} on a contaminated dataset.

%% file: _supplementary_information__ibnn_from_sm.tex
%
%

This section demonstrates that when creating an \ac{IBNN} to solve a given task, one could avoid the training stage by simply transferring the weights from an existing SM network of analogous architecture that has already been trained for the same task. The proof is formulated for the case of a binary classification task and one-hidden-layer networks, but this should not limit its generality. In addition, we show that the magnitude of the parameter $\lambda$ used in \ac{IBNN} has an impact on the accuracy achieved by the network.

Formally, let $\Psi$ be a \ac{SM}-based binary classifier $\Psi(\ve{x}) = H(s(\ve{x}))$, which performs a binary decision, embodied by the Heaviside function $H$, based on the sign of a score $s: \R[D] \to \R$ assigned to each input sample $\ve{x} \in \R[N]$ through a single hidden layer of $D$ \ac{SM} neurons and a \ac{FC}-based head according to
    
\begin{equation}
    \label{eq:suppl_info:eq_score_SM}
    s(\ve{x}) = \veT{c} \, \SM(\ve{x}; \ma{M}, \ve{b}) - c_0 = \sum_{i=1}^D c_i \, \phi(\, y_i(\ve{x}) \,) - c_0 \, ,
\end{equation}

\noindent with $\ve{c} = (c_i)_{i=1}^{D} \in \R[D]$ and $c_0 \in \R$ are the weights of the output layer,
and where

\begin{equation}
    \label{eq:suppl_info:eq_y}
    y_i(\ve{x}) = \veT{m}_i \ve{x} - b_i \, ,
\end{equation}

\noindent where $\ve{m}_i \in \R[N]$ and $b_i \in \R$ are, respectively, the linear filter and the bias for the $i$-th unit of the hidden layer
(see Eq.~\eqref{eq:u} and Eq.~\eqref{eq:methods:def_SM_matrix_form}). And let us assume, without any loss of generality, that $\{\ve{c}, c_0\}$ has been normalized to make $\ve{c}$ a unit vector, that is, $\Vert\ve{c}\Vert_{2}=1$.

We assume that this network has been trained on the set $\mathcal{T}$ composed of two classes $\mathcal{T}_{0}$ and $\mathcal{T}_{1}$ achieving perfect accuracy on it,
that is,

\begin{equation*}
    \Psi(\ve{x}^{(t)}) = \left\{
        \begin{matrix}
            1 \, , & \forall \ve{x}^{(t)} \in \mathcal{T}_{1} \\
            0 \, , & \forall \ve{x}^{(t)} \in \mathcal{T}_{0}
        \end{matrix}
    \right. \; :
\end{equation*}

\noindent in other words, $\SM(\cdot; \ma{M}, \ve{b})$ makes the sets $\SM(\mathcal{T}_{1}; \ma{M}, \ve{b})$ and $\SM(\mathcal{T}_{0}; \ma{M}, \ve{b})$ linearly separable, and $\{\ve{c}, c_0\}$ separates them. Furthermore, we assume that such perfect classifications is achieved through the classifier $s$ of Eq.~\eqref{eq:suppl_info:eq_score_SM} by a margin $M$, that is, that

\begin{equation}
    \label{eq:suppl_info:eq_margin_M}
    \min_{\ve{x}^{(1)} \in \mathcal{T}_{1}} \left\{ s(\ve{x}^{(1)}) \right\}
    \, - 
    \max_{\ve{x}^{(0)} \in \mathcal{T}_{0}} \left\{ s(\ve{x}^{(0)}) \right\}
    \; \geq \; M \; . 
\end{equation}

We define another network $\hat{\Psi}(\ve{x}) = H(\hat{s}(\ve{x}))$, with analogous architecture to $\Psi$ but with \ac{IBNN} instead of \ac{SM} neurons in its hidden layer, and with the peculiarity that the \ac{IBNN} neurons use the same exact $\{\ma{M}, \ve{b}\}$ of the respective \ac{SM} neurons in Eqs.~\eqref{eq:suppl_info:eq_score_SM}-\eqref{eq:suppl_info:eq_y} but whose scoring hyperplane $\{\hat{\ve{c}}, \hat{c}_0\}$ is allowed to differ, and thus

\begin{equation}
    \label{eq:suppl_info:ibnn_score}
    \hat{s}(\ve{x})
        = \hat{\ve{c}}^{T} \IBNN(\ve{x}; \ma{M}, \ve{b}, \lambda) - \hat{c}_0 
        = \sum_{i=1}^D \hat{c}_i \, \phi(\, z_i(\ve{x}) \,) - \hat{c}_0 \, ,
\end{equation}

\noindent where

\begin{equation}
    z_i(\ve{x}) =
    \veT{m}_i \ve{x} - b_i
    \,\mLbin\, \lambda \sum_{k=1}^D w_{ik} \, \sigma\!\left( z_k(\ve{x}) \!-\! z_i(\ve{x}) \right)  \, .
\end{equation}

\noindent We will show that the training error achievable by the \ac{IBNN}-based classifier $\hat{\Psi}$ is below a bound that depends linearly on the magnitude $|\lambda|$ and which could be made, consequently, arbitrarily small. We will do so by
studying the case of $\hat{\ve{c}}=\ve{c}$ and $\hat{c}_0 = c_0$, that is, a scoring hyperplane for $\hat{\Psi}$ identical to that of $\Psi$, as a suboptimal instance of $\{\hat{\ve{c}}, \hat{c}_0\}$.
In the remaining of the section we will use, for convenience, the shorter notation
$\IBNN(\ve{x}) = \IBNN(\ve{x}; \ma{M}, \ve{b}, \lambda)$ and $\SM(\ve{x}) = \SM(\ve{x}; \ma{M}, \ve{b})$ where the parameters of the layer are left implied.

First, a bound on the maximum distance between the images of the maps $\SM(\cdot)$ and $\IBNN(\cdot)$ for the same data point $\ve{x}$ can be established. The difference can be written as

\begin{equation}
    \label{eq:suppl_info:eq_ibnn_minus_sm}
    \IBNN(\ve{x}) - \SM(\ve{x}) =
    \big( \Phi \circ \ma{F} \circ \ma{A} \big)(\ve{x}) - \big( \Phi \circ \ma{A} \big)(\ve{x}) =
    \big(\Phi \circ \ma{F}\big)(\ma{A}(\ve{x})) - \Phi(\ma{A}(\ve{x})) \, ,
\end{equation}

\noindent where the dependence $\ma{F}=\ma{F}(\cdot;\lambda)$ has also left implied. From that form, for any $\Phi=(\phi)_{i=1}^{D}$ contractive like in the case of $\phi=\ReLU$, we have that, for any $\ve{y}=\ma{A}(\ve{x})$,

\begin{equation}
    \label{eq:suppl_info:eq_bound_difference_ibnn_and_sm_without_a}
    \begin{split}
        \left\Vert
            \big(\Phi \circ \ma{F}\big)(\ve{y}) - \Phi(\ve{y})
        \right\Vert_{2}^{2}
        & \leq \,
        \left\Vert
            \ma{F}(\ve{y}) - \ve{y}
        \right\Vert_{2}^{2}
        \\
        & = \,
        \left\Vert \,
            - \lambda \; \ma{B}\big(\ma{F}(\ve{y})\big)
        \, \right\Vert_{2}^{2}
        \, \leq \,
        |\lambda|^{2} \, \left\Vert \,
            \ma{B}\big(\ma{F}(\ve{y})\big)
        \, \right\Vert_{2}^{2}
        \\
        & \leq |\lambda|^{2} \, D
        \; ,
    \end{split}
\end{equation}

\noindent wherein the second line is derived from the very definition of the fixed-point function $\ma{F}$ (see Eq.~\eqref{eq:methods:def_F}) when a unique solution exists, since $\ma{F}(\ve{y}) = \ve{y} - \lambda \ma{B}\big(\ma{F}(\ve{y}))$, and wherein the last inequality derives straightforwardly from the definition of $\ma{B}=(\ma{B}_i)_{i=1}^{D}$ (see Eq.~\eqref{eq:methods:def_Bi}) for a sigmoidal $|\sigma| \leq 1$ and positive and normalized $w_{ik}$ as

\begin{equation}
    \left| \ma{B}_{i}(\ve{z})\right|
    \,=\,
    \left| \sum_{k=1}^{D} w_{ik} \sigma(z_k-z_i) \right|
    \,\leq\,
    \sum_{k=1}^{D} \left| w_{ik} \right| \left| \sigma(z_k-z_i) \right|
    \,\leq\,
    \sum_{k=1}^{D} \left| w_{ik} \right| \leq 1 \, .
\end{equation}

Therefore, since by virtue of Eqs.~\eqref{eq:suppl_info:eq_ibnn_minus_sm}-\eqref{eq:suppl_info:eq_bound_difference_ibnn_and_sm_without_a}

\begin{equation}
    \left\Vert
        \IBNN(\ve{x}) - \SM(\ve{x})
    \right\Vert_{2}
    \,\leq\, |\lambda| \sqrt{D} \, ,
\end{equation}

\noindent we can write, for any $\ve{x}\in\R[N]$,

\begin{equation}
    \IBNN(\ve{x}) =
    \SM(\ve{x})  +  \alpha(\ve{x}) \, \ve{q}(\ve{x})
\end{equation}

\noindent for certain $\alpha(\ve{x})\in\R$ and $\ve{q}(\ve{x}) \in \R[D]$ such that
$0 \leq \alpha(\ve{x}) \leq |\lambda| \sqrt{D}$ and
$\Vert\ve{q}(\ve{x})\Vert_{2}=1$.
Using this fact, and for an identical scoring hyperplane to that of $\Psi$,
that is, $\hat{\ve{c}}=\ve{c}$ and $\hat{c}_0 = c_0$), we have,
$\forall \ve{x}^{(1)} \in \mathcal{T}_{1}$ and $\forall \ve{x}^{(0)} \in \mathcal{T}_{0}$,

\begin{equation}
    \begin{split}
    \hat{s} & (\ve{x}^{(1)}) - \hat{s}(\ve{x}^{(0)})
    \\
    &= \,
    \Big( \ve{c}^{T} \IBNN(\ve{x}^{(1)}) - c_0 \Big) \, - \, 
    \Big( \ve{c}^{T} \IBNN(\ve{x}^{(0)}) - c_0 \Big)
    \\
    &= \,
    \bigg( \ve{c}^{T} \, 
        \Big( \SM(\ve{x}^{(1)})  +  \alpha(\ve{x}^{(1)}) \, \ve{q}(\ve{x}^{(1)}) \Big)
        - c_0
    \bigg)
    -
    \bigg( \ve{c}^{T} \, 
        \Big( \SM(\ve{x}^{(0)})  +  \alpha(\ve{x}^{(0)}) \, \ve{q}(\ve{x}^{(0)}) \Big) 
        - c_0
    \bigg)
    \\
    &= \,
    \bigg(
        \Big( \ve{c}^{T} \, \SM(\ve{x}^{(1)}) - c_0 \Big)
        -
        \Big( \ve{c}^{T} \, \SM(\ve{x}^{(0)}) - c_0 \Big)
    \bigg)
    +
    \ve{c}^{T} \, \Big(
        \alpha(\ve{x}^{(1)}) \, \ve{q}(\ve{x}^{(1)})
        -
        \alpha(\ve{x}^{(0)}) \, \ve{q}(\ve{x}^{(0)})
    \Big)
    \\
    &\geq \;
    M - 2 |\lambda| \sqrt{D} \, ,
    \end{split}
\end{equation}

\noindent where the last inequality results from the assumption of Eq.~\eqref{eq:suppl_info:eq_margin_M} for a margin $M$ in the \ac{SM}-based classifier and from the worst case alignment between the unit vectors $\ve{c}$ and $\ve{q}(\ve{x}^{(i)})$.